\documentclass{sistedes}
\usepackage{amsmath}
\usepackage{subcaption}
\usepackage{booktabs}
\usepackage{tabularx}
\usepackage{mathrsfs}

\usepackage{amssymb}
\usepackage{multirow}
\usepackage{makecell}
\usepackage{nicefrac}
\usepackage{siunitx}
\usepackage{rotating}
\usepackage{multirow}
\usepackage{makecell}
\usepackage[a4paper, margin=1in, top=1.5in, bottom=1.5in]{geometry}

\usepackage{biblatex} 
\usepackage{graphicx}
\usepackage{forest}
\usepackage{tikz-qtree}
\usetikzlibrary{shapes.geometric}
\usepackage{forest}
\forestset{
  declare dimen={fork sep}{0.5em},
  forked edge'/.style={
    edge={rotate/.option=!parent.grow},
    edge path'={(!u.parent anchor) -- (.child anchor)},
  },
  forked edge/.style={
    on invalid={fake}{!parent.parent anchor=children},
    child anchor=parent,
    forked edge',
  },
  forked edges/.style={for nodewalk={#1}{forked edge}},
  forked edges/.default=tree,
}

\title{fKAN: Fractional Kolmogorov-Arnold Networks with trainable Jacobi basis functions}

\author{
Alireza Afzal Aghaei\inst{1}
}

\institute{
Independent Researcher\\
\email{alirezaafzalaghaei@gmail.com}
}
\begin{document}
\maketitle

\small

\begin{abstract}

Recent advancements in neural network design have given rise to the development of Kolmogorov-Arnold Networks (KANs), which enhance speed, interpretability, and precision. This paper presents the Fractional Kolmogorov-Arnold Network (fKAN), a novel neural network architecture that incorporates the distinctive attributes of KANs with a trainable adaptive fractional-orthogonal Jacobi function as its basis function. By leveraging the unique mathematical properties of fractional Jacobi functions, including simple derivative formulas, non-polynomial behavior, and activity for both positive and negative input values, this approach ensures efficient learning and enhanced accuracy. The proposed architecture is evaluated across a range of tasks in deep learning and physics-informed deep learning. Precision is tested on synthetic regression data, image classification, image denoising, and sentiment analysis. Additionally, the performance is measured on various differential equations, including ordinary, partial, and fractional delay differential equations. The results demonstrate that integrating fractional Jacobi functions into KANs significantly improves training speed and performance across diverse fields and applications.
\end{abstract}

\section{Introduction}
The activation function in artificial neural networks is a mathematical construct that parallels the biological concept of an action potential, which determines neuronal activity. Essentially, the activation function applies a mathematical operation to the raw output of neurons within an artificial neural network, producing the final output. The value derived from this activation process is then transmitted to other neurons in the network, emphasizing the significant influence the choice of activation function has on network performance. Over time, various activation functions have been introduced in the literature \cite{aggarwal2018neural, jagtap2022important}, prompting extensive research to understand and evaluate their impact on network dynamics.

The Sigmoid function, a notable non-linear activation function, has been extensively studied for its efficacy in artificial neural networks \cite{han1995influence}. This function is characterized by its differentiability, a feature that aids in weight updates during the backpropagation algorithm, as its derivative is inherently based on the function itself. However, it has unintended side effects, such as the vanishing gradient problem within the network. This has led researchers to propose enhanced versions of the sigmoid function, as explored in \cite{qin2018optimized}. The hyperbolic tangent function, another well-known activation function, shares similar drawbacks, intensifying challenges like the vanishing gradient problem in neural networks \cite{lecun2002efficient}.

In 2010, Nair and Hinton introduced a revolutionary activation function known as the Rectified Linear Unit (ReLU), which has since become the most commonly used activation function in recent years \cite{nair2010rectified}. Numerous studies have investigated the impact of ReLU on the learning process within neural networks, with significant contributions from researchers like Glorot et al. \cite{glorot2011deep} and Ramachandran et al. \cite{ramachandran2017searching}. ReLU, a non-linear activation function, provides differentiability and a derivative formula based on the function itself. However, it has a limitation in that it returns 0 for negative inputs, leading to ‘dead neurons’ in the network, which obstructs the acquisition of meaningful knowledge. To address this issue, researchers have proposed improved versions of ReLU, including Leaky ReLU \cite{maas2013rectifier}, Parametric ReLU \cite{he2015delving}, Gaussian Error Linear Units (GELU) \cite{hendrycks2016gaussian}, Exponential Linear Units (ELU) \cite{clevert2015fast}, and Mish \cite{misra2019mish}.

Building on the foundation laid by ReLU and its variants, the concept of trainable or adaptive activation functions has emerged as a promising avenue in neural network research. These functions, unlike their fixed counterparts, can adjust their behavior based on the data they process, potentially leading to more efficient and effective learning processes. Adaptive activation functions, such as the Parametric ReLU introduced by He et al. \cite{he2015delving}, the Learnable Extended Activation Function proposed by Bodyanskiy et al. \cite{bodyanskiy2023learnable}, modReLU developed by Arjovsky et al. \cite{arjovsky2016unitary}, and DELU explored by Pishchik \cite{pishchik2023trainable}, allow the parameters of the activation function to be learned during the training process, adapting to the specific characteristics of the input data. This adaptability can mitigate some of the inherent limitations of fixed activation functions, such as the 'dead neuron' problem associated with ReLU, by ensuring that the activation function evolves in tandem with the network's learning trajectory. The exploration of such functions is driven by the hypothesis that a network's capacity to learn can be significantly improved by allowing more aspects of the model, including the activation function, to be subject to optimization.

In recent developments, the emergence of Kolmogorov-Arnold Neural Networks has been noted \cite{liu2024kan}. These networks represent an innovative category within neural network architectures, utilizing the mathematical characteristics of splines and B-splines activation functions to achieve precise data approximation. Inspired by the Kolmogorov-Arnold representation theorem, KANs replace traditional linear weights with trainable 1D functions parametrized as splines. This substitution allows KANs to capture complex patterns without the need for linear weights, leading to a more flexible and interpretable model. The structure of KANs offers several advantages over conventional Multi-Layer Perceptrons (MLPs). 

These networks can achieve comparable or superior accuracy in tasks such as data fitting and solving partial differential equations with significantly smaller network sizes \cite{liu2024kan, cheon2024kolmogorov, liu2024ikan}. Moreover, the use of splines as activation functions on edges enhances the interpretability of the network, making it easier to visualize and understand the decision-making process \cite{liu2024kan}. Empirical studies have also shown that KANs exhibit faster neural scaling laws than MLPs, suggesting that they can scale more efficiently with increasing data size \cite{liu2024kan, genet2024temporal, abueidda2024deepokan}.  

The field of KANs has seen significant advancements since the introduction of the KAN architecture by the B-splines. This architecture has paved the way for the development of various types of KANs, each with its unique features and applications. One such extension is the Wavelet KANs \cite{bozorgasl2024wav}, introduced by Bozorgasl and his team. This innovative approach incorporates orthogonal Wavelets into the KAN architecture. The integration of Wavelets has been instrumental in enhancing the performance of KANs, while also improving their interpretability. Another noteworthy development in the field is the Fourier KANs \cite{xu2024fourierkan} by Xu and colleagues. This model was specifically designed for efficient graph-based recommendation systems within a graph collaborative filtering framework. The Fourier KANs have proven to be highly effective in making accurate recommendations, thereby enhancing the user experience on various digital platforms. Similarly, Genet et al. introduced Temporal KANs, a recurrent neural network architecture inspired by Long Short Term Memory (LSTM) networks, as outlined in \cite{genet2024tkan}. Temporal KANs excel in handling sequential data tasks like time-series analysis and natural language processing, capitalizing on the strengths of LSTM networks. Another notable advancement in the context of KANs is the development of deep operator networks for mechanical problems, discussed in \cite{abueidda2024deepokan}. Additionally, Samadi et al. conducted a comprehensive analysis of KANs, focusing on their smoothness properties, in their work on smooth KANs \cite{samadi2024smooth}.

In numerical analysis \cite{singh2012approximation, gautschi2004orthogonal,shen2011spectral}, different basis functions have been utilized to approximate functions. Splines, Fourier basis, rational functions, Wavelets, Binary-valued functions, and polynomials are some of the most common options. Fourier functions are one of the best choices for problems with periodic oscillations. Splines, despite their inherent intricacies, are great options for approximating complex shapes because of their flexibility. These intricacies include piece-wise nature, continuity issues, and complex implementation. On the other hand, polynomials are an alternative option when it comes to approximating functions. Polynomials offer simpler implementation, continuity, and global adaptation. However, they suffer of Runge's phenomenon, numerical instability, and even overfitting. Moreover, the nature of these basis functions does not allow for an accurate approximation of an unknown function with a low degree. 

Orthogonal polynomials are a set of polynomials introduced to mitigate the limitations of raw polynomials. These functions build a complete set of basis functions and have the orthogonality property, which ensures that each pair of functions in the set is perpendicular to each other \cite{shen2011spectral}. This property ensures various mathematical theorems and corollaries that mitigate problems such as Runge's phenomenon. Among the emerging trends in the field of neural network activation functions, the use of orthogonal polynomials has attracted significant attention in the literature. For example, Deepthi et al. \cite{deepthi2023development} developed the Chebyshev activation function for classification and image denoising, Alsaadi et al. \cite{alsaadi2023control} developed a Chebyshev neural network to control a hydraulic generator regulating system, Venkatappareddy et al. \cite{venkatappareddy2021legendre} modeled the max pooling layer with Legendre polynomials, Wang et al. \cite{wang2020wind} developed a hybrid neural network based on Laguerre polynomials for wind power forecasting, Patra et al. \cite{patra2010development} utilized Laguerre neural networks for wireless sensor networks, Ebrahimi et al. \cite{ebrahimzadeh2016classification} classified ECG signals based on Hermite functions and neural networks, Guo et al. \cite{guo2023graph} used Bernstein polynomials in graph neural networks. Additionally, Sidharth \cite{ss2024chebyshev} introduced Chebyshev KANs, showcasing another innovative application of polynomial-based neural network architectures.

All attempts and explorations to find an optimal activation function, possibly piece-wise, rely on certain mathematical functions that appear to be a good fit for some physical or engineering problems. However, most of these functions, especially orthogonal polynomials, are interconnected. In fact, all these functions are subsets of the well-known Generalized hypergeometric function \cite{carlson1991b, shen2011spectral, forster2011splines,manni2017generalized}. This parametric function can be reduced to any of the aforementioned orthogonal activation functions and yield the same results. To illustrate these relationships more clearly, we have plotted them in Figure \ref{fig:hypergeometric-hierarchy} which also depicts a literature review of neural network activation functions. Based on the hierarchy depicted in this figure, one can obtain Chebyshev polynomials by fixing parameters of the Gegenbauer, Jacobi, or hypergeometric function.

As observed, the majority of research papers concentrate on specific instances of the hypergeometric function. This can be directly attributed to the unique properties of these functions. For instance, Legendre polynomials possess a simple orthogonality property that can expedite certain calculations, while Chebyshev polynomials have theoretical foundations that make them the most accurate function approximations. Laguerre and Hermite polynomials can manage input values beyond a finite domain, a prerequisite for Chebyshev and Legendre polynomials. Additionally, the rapid computation of these functions’ derivatives is another property that needs consideration. Specifically, the leaves in Figure \ref{fig:hypergeometric-hierarchy} are very special functions with numerous mathematical properties, while nodes near the root have fewer mathematical properties but more general definitions.

{\large
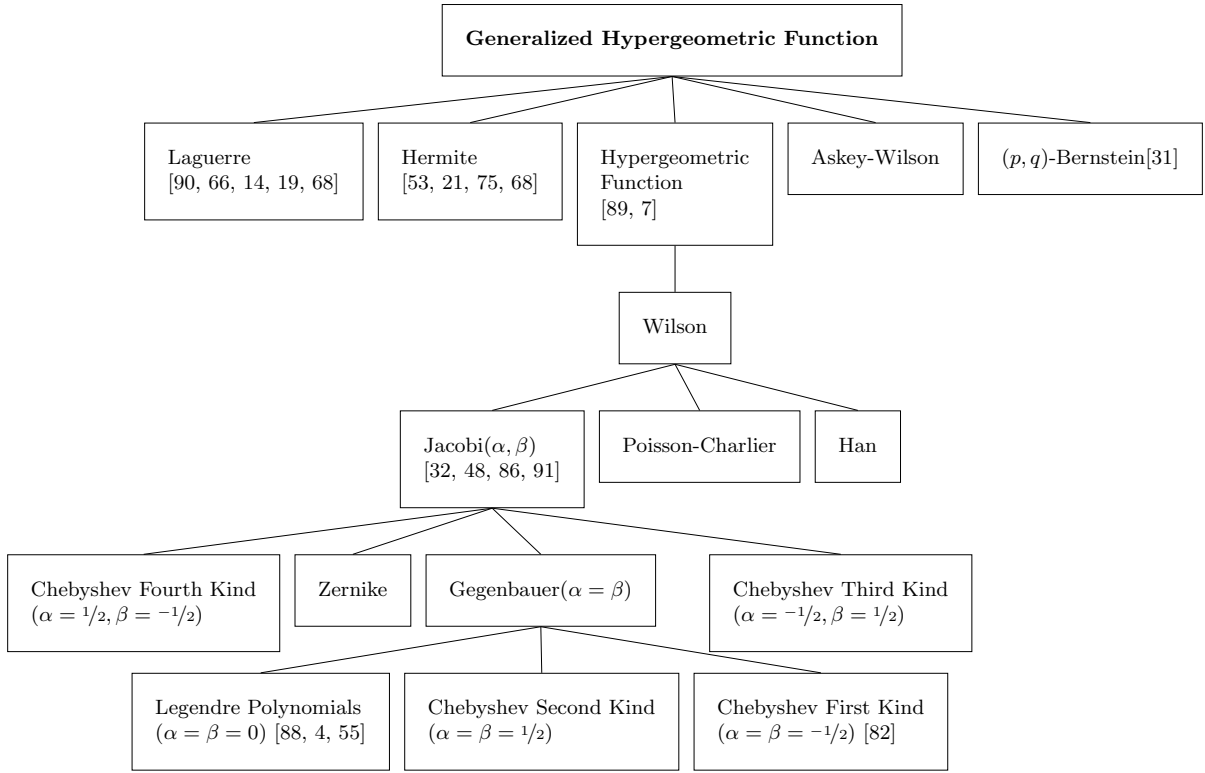
\begin{figure}[ht]
    \centering
    \resizebox{1\textwidth}{!}{
\begin{forest}
 for tree={grow'=-90, draw, align=left, inner sep=10pt, l=20pt, l sep=20pt},
  forked edges,
 [\textbf{Generalized Hypergeometric Function}
 [{ $(p, q)$-Bernstein\cite{guo2023graph}}
 ]
 [{Askey-Wilson}
 ]
 [Hypergeometric\\Function\\
 \cite{vieira2023bicomplex, benzoubeir2009hypergeometric}
 [Wilson
 [Han]
 [Poisson-Charlier]
 [{Jacobi($\alpha, \beta$) \\\cite{hadian2020simulation,liu2021jacobi,tao2023longnn, wang2022powerful}}
 [{Chebyshev Third Kind\\($\alpha=\nicefrac{-1}{2},\beta=\nicefrac{1}{2}$)}]
 [{Gegenbauer($\alpha=\beta$)}
 [{Chebyshev First Kind\\ ($\alpha=\beta=\nicefrac{-1}{2}$) \cite{ss2024chebyshev}}]
 [{Chebyshev Second Kind\\($\alpha=\beta=\nicefrac{1}{2}$)}]
 [{Legendre Polynomials\\($\alpha=\beta=0$)  \cite{venkatappareddy2021legendre,alsaedi2022classification,mall2016application}}]]
 [{Zernike}]
 [{Chebyshev Fourth Kind\\($\alpha=\nicefrac{1}{2},\beta=\nicefrac{-1}{2}$)}]
 ]
 ]
 ]
 [{Hermite\\\cite{ma2005constructive, ebrahimzadeh2016classification, rigatos2006feed,peterson2009image}}
 ]
 [Laguerre\\ \cite{wang2020wind,patra2010development,chen2021numerical,dorado2020orthogonal,peterson2009image}
 ]
 ]
 \end{forest}
}
    \caption{This figure depicts the hierarchical classification of the generalized hypergeometric function. It shows that Chebyshev and Gegenbauer polynomials are particular cases within the more general class of Jacobi polynomials, which itself is a subclass of the hypergeometric function.}
    \label{fig:hypergeometric-hierarchy}
\end{figure}
}

Regardless of these interesting properties, orthogonal polynomials are still polynomials and inherit the limitations of polynomials. Fractional orthogonal functions are an extension of orthogonal polynomials that map their functional space from polynomials to fractional functions \cite{rad2023learning}. These functions are commonly used in numerical analysis and simulating differential equations \cite{parand2017accurate,parand2016solving}. However, in some cases, they have been used in neural networks \cite{hajimohammadi2021fractional}. This mapping has been developed for various functions including B-splines and Wavelets \cite{unser2000fractional, blu2000fractional}, Bernoulli functions \cite{parand2017numerical}, Jacobi functions \cite{parand2017fractional}, Legendre polynomials \cite{parand2021parallel}, and Lagrange functions \cite{delkhosh2021new}. The results of these approximations consistently indicate that introducing a fractional order parameter to the basis functions leads to more accurate predictions, even when the fractional order parameter itself is to be determined.

In this paper, we focus on a well-known orthogonal polynomial in the middle of the tree (Figure \ref{fig:hypergeometric-hierarchy}), the Jacobi polynomial, to propose a new trainable activation function for KANs. We demonstrate how the parameters of Jacobi polynomials, shaping the activation function, can be tuned during the training process to enhance neural network performance. Additionally, we explore how the fractional order of the Jacobi functions can be leveraged to further increase the accuracy of the network. While the Jacobi polynomial has been utilized in various neural network architectures by researchers such as Tao et al. \cite{tao2023longnn} and Wang et al. \cite{wang2022powerful}, who employed these functions in Graph neural networks, or Hadian et al. \cite{hadian2020simulation} and Liu et al. \cite{liu2021jacobi}, who employed these functions in physics-informed neural network architectures for the solution of differential equations, this paper aims to provide a novel perspective on orthogonal Jacobi functions. We introduce a fractional Jacobi neural block that can be integrated anywhere in a deep or KAN architecture.

To be more specific, we delve into the effectiveness of fractional Jacobi functions as activation functions, comparing them to common alternatives in neural networks. Subsequently, a novel neural network block is introduced, enabling the use of the proposed activation function in different neural network architectures. Then, we assess the accuracy of the presented method in different deep learning benchmark tasks. We will employ this neural architecture for approximating synthetic regression data, then we use it to classify the MNIST digit classification dataset, denoise the Fashion MNIST dataset, and finally conduct a sentiment analysis task on the IMDB dataset. To further evaluate the proposed neural network, we employ it to approximate the solution of some differential equations. We examine the standard form of the Lane-Emden equation, the one-dimensional Burgers partial differential equation, and a delay differential equation with the Caputo fractional derivative.

The following sections of the paper are structured as follows. First, we explain activation functions and their impact on MLPs and KANs. Then, we explore the characteristics of Jacobi polynomials, highlighting their intrinsic properties. In this section, we discuss the advantages that make Jacobi polynomials well-suited as activation functions in neural networks. Next, we present the methodology proposed in this research by detailing the development of a fKAN. Section 4 validates the proposed method through experiments on real-world problems. Finally, the concluding section offers some final remarks.

\section{Background}
In this section, we provide an overview of MLPs and KANs, focusing on the role of activation functions within these networks. We then explore the key characteristics that define an effective activation function. Following this, we discuss Jacobi polynomials and their relationship to other orthogonal functions, highlighting their suitability for neural networks.

\subsection{Multi-Layer Perceptron}
A Multi-Layer Perceptron, grounded in the universal approximation theory \cite{pinkus1999approximation, nishijima2021universal, kumagai2020universal}, aims to find a weighted summation of various basis functions. Specifically, an approximation to function $\chi(\cdot)$ can be done using:
\begin{equation*}
    \chi(\boldsymbol\zeta) \approx \hat{\chi}(\boldsymbol\zeta) = \sigma\left(b+\sum_{q=0}^Q \theta_q \boldsymbol\zeta_q\right).
\end{equation*}
The universal approximation theory states that a network with a single hidden layer containing a finite number of neurons, can approximate continuous functions on compact subsets of $\mathbb{R}^d$, under mild assumptions on the activation function $\sigma(\cdot)$. In a mathematical sense, it ensures that exists some $\theta$ and $b$ such that:
\begin{equation*}
    \sup_{\boldsymbol\zeta \in \mathbb{R}^d} \| \chi(\boldsymbol\zeta) - \hat{\chi}(\boldsymbol\zeta) \| < \epsilon.
\end{equation*}
In practice, the accuracy of this estimation improves as a result of the non-linear nested composition of these approximations \cite{gripenberg2003approximation,lu2017expressive,kidger2020universal}. Mathematically, an MLP can be represented as:
\begin{equation*}
    \begin{aligned}    
	\mathcal{H}_0&=\boldsymbol{\zeta}, &\boldsymbol{\zeta} \in \mathbb{R}^d, \\
	\mathcal{H}_i&=\sigma_i(\theta^{(i)}\mathcal{H}_{i-1}+b^{(i)}), & i =1,2,\ldots L-1,\\
	\mathcal{H}_L&=\psi(\theta^{(L)}\mathcal{H}_{L-1}+b^{(L)}).
\end{aligned}
\end{equation*}
In this definition, $\boldsymbol{\zeta}$ represents the input sample fed into the input layer $\mathcal{H}_0$. The output of the $i$-th layer, $\mathcal{H}_i$, is computed based on the output of the previous layer, $\mathcal{H}_{i-1}$. This computation involves matrix multiplication between the learnable weights $\theta^{(i)}$ and the output of the previous layer, $\mathcal{H}_{i-1}$. This weight matrix can be either dense or possess a special structure, such as a Toeplitz matrix, which is commonly used in convolutional neural networks (CNNs). Additionally, a bias term $b^{(i)}$ is added to this result to shift the activation function and better fit the data. In this scenario, a non-linear activation function $\sigma_i(\cdot)$ is applied to the result to overcome the limitations of linear combinations and capture more complex patterns within the data. The output layer of the network, $\mathcal{H}_L$, is computed using a final activation function, which is typically a linear function for regression tasks and a Sigmoid or softmax function for classification tasks.

Following the forward pass, the neural network's overall error on the training data is computed. During the backpropagation phase, the gradients of the error with respect to each weight are calculated and used to update the weights, thereby minimizing the error. In this phase, the derivative of the activation function plays a crucial role, as it is used to adjust the network weights in a manner that reduces the overall network error. Therefore, one of the most significant properties of an activation function is its differentiability. According to \cite{jagtap2022important}, an effective activation function should possess the following properties:
\begin{itemize}
\item Differentiability: During the optimization process, the ability to compute the activation function's derivative is vital. Differentiability ensures the smooth updating of weights; without it, the process is disrupted.
\item Non-linearity: Activation functions generally fall into two categories - linear and non-linear. Networks with multiple layers of linear or identical activation functions behave similarly to single-layer networks with the same characteristics. Therefore, non-linear activation functions are particularly valuable, as they enable the network to extract complex information from data, enhancing the network's capacity for learning.
\item Finite Range/Boundedness: Since the activation function determines a neuron's final output, limiting the output values to a specific interval enhances network stability.
\item Vanishing and Exploding Gradient Issues: In the backpropagation process, the chain rule is employed to adjust weights. In essence, the gradients in the initial layers are derived from the product of gradients in the final layers. If these generated gradients become exceedingly small or excessively large, the challenges of gradient vanishing or exploding arise, hindering the correct weight updates in the initial layers.
\item Computational Efficiency: The activation function's formulation should be structured to ensure that both the forward and backward processes in the network do not involve intricate or resource-intensive computations.
\end{itemize}

\subsection{Kolmogorov-Arnold Networks}
In contrast to MLPs, KANs are based on the Kolmogorov-Arnold representation theorem, which states that a multivariate continuous function can be expressed as a finite composition of continuous functions of a single variable and the binary operation of addition \cite{braun2009constructive,kolmogorov1957representation}. Mathematically:
\begin{equation*}
    \chi(\boldsymbol\zeta) \approx \hat{\chi}(\boldsymbol\zeta) =\sum_{q=0}^{2d}\Phi_q\left(\sum_{p=1}^d \phi_{q,p}(\boldsymbol{\zeta}_p)\right),
\end{equation*}
for some $d$ with trainable functions $\Phi_{q}:\mathbb{R}\to \mathbb{R}$, $\phi_{q,p}:[0,1]\to \mathbb{R}$ and known function $\chi(\boldsymbol\zeta) : [0,1]^d\to\mathbb{R}$. As proposed by the KAN \cite{liu2024kan}, one may consider B-spline \cite{forster2011splines, manni2017generalized} curves combined with the SiLU activation function \cite{paul2022sinlu} for these functions. Regardless of the choices of these functions, in matrix form, we can formulate the KANs as:
\begin{equation*}
    \hat{\chi}(\boldsymbol\zeta)={\bf \Phi}_{n-1}\circ\cdots \circ{\bf \Phi}_1\circ{\bf \Phi}_0\circ {\boldsymbol{\zeta}},
\end{equation*}
where ${\bf \Phi}_{p,q} = \phi_{p,q}(\cdot)$. The convergence of this approximation can be seen as:
\begin{equation*}
    \max_{|\alpha| \le m}\sup_{\boldsymbol\zeta \in [0,1]^d} \|\mathfrak{D}^\alpha \left\{\chi(\boldsymbol\zeta) - \hat{\chi}(\boldsymbol\zeta)\right\}  \| < \epsilon,
\end{equation*}
where $\mathfrak{D}$ is the derivative operator, and $\epsilon$ is a constant based on the behavior of $\chi(\cdot)$ and trainable functions $\phi_{p,q}(\cdot)$. The complete proof of the convergence of this approximation can be found in \cite{liu2024kan, braun2009constructive, dzhenzher2021structured, schmidt2021kolmogorov}.

As a consequence of this theorem and its utility, various extensions have been proposed to facilitate faster implementations. For example, Lorentz \cite{lorentz1962metric} proved that the functions $\Phi_q(\cdot)$ can be reduced to a single function $\Phi(\cdot)$:
\begin{equation*}
     \hat{\chi}(\boldsymbol\zeta) =\sum_{q=0}^{2d}\Phi\left(\sum_{p=1}^d \phi_{q,p}(\boldsymbol{\zeta}_p)\right).
\end{equation*}
In another work, Sprecher \cite{sprecher1965structure} showed that the functions $\phi_{p,q}$ can be replaced by a single function $\phi(\cdot)$:
\begin{equation*}
     \hat{\chi}(\boldsymbol\zeta) =\sum_{q=0}^{2d}\Phi\left(q+\sum_{p=1}^d \theta_p \phi(\boldsymbol{\zeta}_p+\nu q) \right).
\end{equation*}
Both of these variants aimed to provide an easier approach for implementation. However, they have not yet addressed the issues of the Kolmogorov-Arnold representation, which include the non-smoothness of the inner functions and the time complexity of KANs \cite{girosi1989representation}.

Recently, some research has proposed replacements for B-splines and alternative formulations of KANs to potentially achieve more accurate solutions with lower time complexity. For example, Fourier KANs \cite{xu2024fourierkan}, Wavelet KANs \cite{bozorgasl2024wav}, and radial basis function (RBF) KANs \cite{li2024kolmogorov}. Another approach is Chebyshev KANs, which utilize orthogonal polynomials as the univariate functions in KANs \cite{ss2024chebyshev, SynodicMonth2024}. This architecture is defined as:
\begin{equation*}
     \hat{\chi}(\boldsymbol\zeta) =\sum_{q=1}^{n}\sum_{p=1}^d \theta_{p,q}T_q(\tanh(\boldsymbol{\zeta}_p)).
\end{equation*}
Here, $n$, a hyperparameter, represents the degree of the Chebyshev polynomials, $T_q(\cdot)$ denotes the Chebyshev polynomial of degree $q$, and $\theta_{p,q}$ are the trainable weights. This approach is a counterpart of KANs, named Learnable Activation Networks (LANs).

Although the Chebyshev KAN demonstrated acceptable accuracy while maintaining low computational complexity \cite{ss2024chebyshev, SynodicMonth2024}, it still inherits the limitations associated with polynomials as discussed in the introduction. In the following section, we introduce fractional Jacobi functions as a potential replacement for Chebyshev polynomials in KANs.

\subsection{Jacobi polynomials}

The Jacobi polynomials $\mathcal{J}_n^{(\alpha,\beta)}(\zeta)$ are a special case of hypergeometric functions. The generalized form of the hypergeometric function ${}_p\mathcal{F}_q : \mathbb{C}^{p+q+1} \to \mathbb{C}$ can be represented as:
\begin{equation*}
    {}_p\mathcal{F}_q\left(\begin{matrix}a_1, a_2, \ldots, a_p \\ b_1, b_2, \ldots, b_q\end{matrix}; \zeta\right) = \sum_{n=0}^{\infty} \frac{(a_1)_n (a_2)_n \ldots (a_p)_n}{(b_1)_n (b_2)_n \ldots (b_q)_n} \frac{\zeta^n}{n!},
\end{equation*}
where $(a)_n$ is the rising sequential product, known as the Pochhammer symbol, defined by:
\begin{equation*}
    (a_i)_n = \frac{\Gamma(a_i+n)}{\Gamma(a_i)} = a_i(a_i+1)(a_i+2)\ldots(a_i+n-1), \quad n>0,
\end{equation*}
in which $\Gamma(\xi) = \int_{0}^{\infty} s^{\xi-1}\exp(-s)\text{d}s$ is the Gamma function \cite{shen2011spectral}. In the case of $n=1$, this empty product is defined as one \cite{benzoubeir2009hypergeometric}. Vieira \cite{vieira2023bicomplex} proved that the infinite hypergeometric series converges if any of the following conditions are met:
\begin{align*}
& \text{(a)} \quad p \leq q, \\
& \text{(b)} \quad q = p - 1 \land |\zeta| < 1, \\
& \text{(c)} \quad q = p - 1 \land \mathscr{R}\left(\sum_{i=1}^{p-1} (b_i - p) - \sum_{j=1}^{q} a_j\right) > 0 \land |\zeta| = 1.
\end{align*}
For $p=2$ and $q=1$, this analytic function simplifies to the Gaussian hypergeometric function:
\begin{equation*}
    {}_2\mathcal{F}_1\left(\begin{matrix}a_1, a_2 \\ b_1\end{matrix}; \zeta\right) = \sum_{n=0}^{\infty} \frac{(a_1)_n (a_2)_n}{(b_1)_n} \frac{\zeta^n}{n!}.
\end{equation*}
Similarly, for $p=0$ and $q=1$, the generalized hypergeometric function transforms into the Bessel function:
\begin{equation*}
{\displaystyle J_{\alpha }(\zeta)={\frac {\left({\frac {\zeta}{2}}\right)^{\alpha }}{\Gamma (\alpha +1)}}\;_{0}\mathcal{F}_{1}\left(\alpha +1;-{\frac {\zeta^{2}}{4}}\right)}.
\end{equation*}
The analytical function ${}_p\mathcal{F}_q$ with any non-positive values of $a_i$ reduces to a polynomial of degree $-a_i$ in the argument $\zeta$. For example, Askey–Wilson polynomials, Wilson polynomials, continuous Hahn polynomials, and Jacobi polynomials are some such cases.

For the specific choice $p=2, q=1, a_1=-n, a_2=n+1+\alpha+\beta$, and $b_1=1+\alpha$ where $\alpha, \beta > -1$, the Gaussian hypergeometric series reduces to the Jacobi polynomials:
\begin{equation*}
\mathcal{J}_{n}^{{(\alpha ,\beta )}}(\zeta)={\frac  {(\alpha +1)_{n}}{n!}}\,{}_2\mathcal{F}_1\left(\begin{matrix}-n,n+1+\alpha +\beta \\ 1+\alpha\end{matrix}; {\tfrac  {1}{2}}(1-\zeta)\right).
\end{equation*}
These polynomials appear in various scientific fields such as signal processing, numerical analysis, machine learning, economics, and quantum mechanics \cite{rad2023learning}. Hence, they can be defined through several definitions. For example, they are the eigenfunctions of $\mathcal{L}(\mathcal{J}^{(\alpha,\beta)}_n) = \lambda_n \mathcal{J}^{(\alpha,\beta)}_n$ for the linear second-order Sturm–Liouville differential operator:
\begin{equation*}
\begin{aligned}
\mathcal{L}( \mathcal{J}^{(\alpha,\beta)}_n) := &-(1-\zeta)^{-\alpha}(1+\zeta)^{-\beta}\frac{\text{d}}{\text{d}\zeta}\big((1-\zeta)^{\alpha+1}(1+\zeta)^{\beta+1}\frac{\text{d}}{\text{d}\zeta}\mathcal{J}^{(\alpha,\beta)}_n(\zeta)\big) \\
&=(\zeta^2-1)\frac{\text{d}^2}{\text{d}\zeta^2}\mathcal{J}^{(\alpha,\beta)}_n(\zeta)+\left(\alpha-\beta+(\alpha+\beta+2)\zeta\right)\frac{\text{d}}{\text{d}\zeta}\mathcal{J}^{(\alpha,\beta)}_n(\zeta),
\end{aligned}
\end{equation*}
with the eigenvalues $\lambda_n = n (n+1+\alpha+\beta)$. Alternatively, the following three-term recurrence formula can generate the Jacobi polynomials:
\begin{equation*}
    \begin{aligned}
    \mathcal{J}^{(\alpha,\beta)}_{0}(\zeta) &= 1, \\
    \mathcal{J}^{(\alpha,\beta)}_{1}(\zeta) &= A_0 \zeta +B_0,\\
    \mathcal{J}^{(\alpha,\beta)}_{n+1}(\zeta) &=(A_{n}\zeta+B_{n})\mathcal{J}^{(\alpha,\beta)}_{n}(\zeta)-C_{n}\mathcal{J}^{(\alpha,\beta)}_{n-1}(\zeta), \quad n\ge 1,
\end{aligned}
\end{equation*}
where
\begin{equation*}
    \begin{aligned}
A_n &= \dfrac{(2n+\alpha+\beta+1)(2n+\alpha+\beta+2)}{2(n+1)(n+\alpha+\beta+1)},\\
B_n &= \dfrac{(\alpha^{2}-\beta^{2})(2n+\alpha+\beta+1)}{2(n+1)(n+\alpha+\beta+1)(2n+\alpha+\beta)},\\
C_n &= \dfrac{(n+\alpha)(n+\beta)(2n+\alpha+\beta+2)}{(n+1)(n+\alpha+\beta+1)(2n+\alpha+\beta)}.
\end{aligned}
\end{equation*}
Choosing different values for $\alpha$ and $\beta$ results in specific well-known orthogonal polynomials that appear in different engineering applications. Figure \ref{fig:hypergeometric-hierarchy} depicts some of the most used functions in this hierarchy. In the following, we focus on the Jacobi polynomials and their properties which are used in the following sections.
\begin{theorem}
The Jacobi polynomials with $\alpha, \beta > -1$ are a set of orthogonal functions on the interval $[-1,1]$:
\begin{equation*}
    \langle \mathcal{J}_{m}^{(\alpha ,\beta )}, \mathcal{J}_{n}^{(\alpha ,\beta )}\rangle = \int_{-1}^{1} 
\mathcal{J}_{m}^{(\alpha ,\beta )}(\zeta) \mathcal{J}_{n}^{(\alpha ,\beta )}(\zeta)  (1-\zeta)^\alpha(1+\zeta)^\beta \text{d}\zeta = \| \mathcal{J}_{n}^{(\alpha ,\beta )} \|_2 \delta_{m,n},
\end{equation*}
where $\langle f,g\rangle = \int_\Omega f(\xi)g(\xi)w(\xi) \text{d}\xi$ is the inner product operator over the domain $\Omega \subseteq \mathbb{R}$ \cite{shen2011spectral}. \label{thm:domain}
\end{theorem}
\begin{theorem}
The Jacobi polynomials have the following symmetry on the interval $[-1,1]$ \cite{shen2011spectral}:
\begin{equation*}
    \mathcal{J}_{n}^{{(\alpha ,\beta )}}(-\zeta)=(-1)^{n}\mathcal{J}_{n}^{{(\beta ,\alpha )}}(\zeta).
\end{equation*}    
\end{theorem}

\begin{corollary}
The Jacobi polynomials at the boundary points can be approximated by:
\begin{equation*}
    \begin{aligned}
\mathcal{J}_{n}^{{(\alpha ,\beta )}}(-1)&=\frac{(-1)^{n}}{n!}\frac{\Gamma(n+1+\beta)}{\Gamma(1+\beta)}, & \\
\mathcal{J}_{n}^{{(\alpha ,\beta )}}(+1)&\approx n^\alpha & \  n \gg 1.
\end{aligned}
\end{equation*}
\end{corollary}

\begin{theorem}
The derivatives of the Jacobi polynomials can be expressed in terms of themselves. Specifically:
   \begin{equation*}
       {\frac {\text{d}^{m}}{\text{d}\zeta^{m}}}\mathcal{J}_{n}^{(\alpha ,\beta )}(\zeta)={\frac {\Gamma (n+m+1+\alpha +\beta)}{2^{m}\Gamma (m+1+\alpha +\beta)}}P_{n-m}^{(\alpha +m,\beta +m)}(\zeta).
   \end{equation*}
\end{theorem}
\begin{proof}
This property is a direct consequence of the orthogonality and the derivatives of the hypergeometric function:
    \begin{equation*}
        {\frac {\text{d}^{m}}{\text{d}\zeta^{m}}}\ {}_2\mathcal{F}_1\left(\begin{matrix}a_1, a_2 \\ b_1\end{matrix}; \zeta\right)={\frac {(a_1)_{m}(a_2)_{m}}{(b_1)_{m}}}{}_2\mathcal{F}_1\left(\begin{matrix}a_1+m, a_2+m \\ b_1+m\end{matrix}; \zeta\right)
    \end{equation*}
\end{proof}
\begin{theorem}
    \label{thm:roots}
    Jacobi polynomial of the degree $n$ has $n$ real distinct roots    \cite{arvesu2021zeros,Schweizer2021}.
\end{theorem}
\begin{theorem}
   The linear bijective mapping
   \begin{equation*}
       \varphi_{(d_0,d_1)}(\zeta) = \frac{2\zeta-d_0-d_1}{d_1-d_0},
   \end{equation*}
transforms the properties of the Jacobi polynomials from the interval $[-1,1]$ into the desired domain $[d_0,d_1]$. Moreover, for extending the polynomial space into fractional-order functions, one may use:
   \begin{equation*}
       \varphi_{(d_0,d_1, \gamma)}(\zeta) = \frac{2\zeta^\gamma-d_0-d_1}{d_1-d_0}.
   \end{equation*}
Employing these mappings, the shifted fractional Jacobi functions can be generated using:
    \begin{equation*}
        \mathcal{P}^{(\alpha,\beta)}_n(\zeta) = \mathcal{J}^{(\alpha,\beta)}_n(\varphi_{(d_0,d_1, \gamma)}(\zeta)).
    \end{equation*}
\end{theorem}

Examining the characteristics of Jacobi orthogonal polynomials as a potential activation function reveals their non-linearity and differentiability. Their derivatives can be easily computed using a formula derived from the function itself, simplifying calculations during backpropagation and reducing computational burden. Additionally, Jacobi polynomials inherently confine their output within a predetermined range, promoting network stability. Consequently, they offer promising features for activation functions in neural networks. With these insights, the following section introduces an architecture for incorporating Jacobi functions into neural network frameworks.

\section{Fractional KAN}
In this section, we present a novel block structure enabling the utilization of fractional Jacobi functions as activation functions within a neural network. This framework, empowers deep learning architectures or even KANs to determine optimal values for Jacobi polynomial parameters, namely $\alpha$, $\beta$, and the fractional order $\gamma$. Following this, we will explore the intricacies of the Jacobi neural block integrated into the design of the neural network models employed in our experiments.

It's crucial to acknowledge that Jacobi polynomials, as per Theorem \ref{thm:domain}, require their input values to reside within a specific interval to yield meaningful output. However, the output produced by a neural network layer doesn't inherently adhere to this constraint, making it impractical to directly employ Jacobi polynomials as activation functions immediately after a fully connected layer. To tackle this challenge, two different approaches have been proposed: 1) Utilizing batch normalization and 2) Applying a bounded activation function \cite{ss2024chebyshev}. 

The first option introduces certain complications, including increased time complexity \cite{garbin2020dropout}, incompatibility with small batch sizes or online learning \cite{lange2022batch}, and issues with saturating non-linearities \cite{ioffe2015batch, bjorck2018understanding}. On the other hand, the second approach offers lower computational complexity compared to batch normalization. It facilitates optimal backward computations, boasts easy implementation, and ensures consistent behavior in both the training and testing phases. In contrast, the first option exhibits varying behaviors between these phases. Therefore, in this paper, we opt for the second approach to address this case.

The choice of an activation function to confine the output of a layer within a bounded interval ($[d_0, d_1]$) can include options like the Sigmoid function, hyperbolic tangent, Gaussian function, or any bounded-range activation function. While the selection of this function is arbitrary, in this paper, we opt for the Sigmoid function ($\sigma(\zeta) = \frac{1}{1+\exp(-\zeta)}$) due to its range, which spans $(0,1)$. This property of positivity enables the computation of fractional powers without encountering complex-valued numbers in fKAN.

To propose the fractional KAN, we first revisit the original KAN:
\begin{equation*}
     \hat{\chi}(\boldsymbol\zeta) =\sum_{q=0}^{2d}\Phi_q\left(\sum_{p=1}^d \phi_{q,p}(\boldsymbol{\zeta}_p)\right).
\end{equation*}
In general, the fractional KAN can be obtained using a fractional univariate function $\phi_{p,q}(\cdot)$ such as fractional Jacobi functions, fractional Bernoulli functions, or even fractional B-splines \cite{pitolli2018fractional}. This extension can be formulated as:
\begin{equation*}
     \hat{\chi}(\boldsymbol\zeta) =\sum_{q=0}^{2d}\Phi^{(\gamma)}_q\left(\sum_{p=1}^d \phi^{(\gamma)}_{q,p}(\boldsymbol{\zeta}_p)\right),
\end{equation*}
where $\gamma$ is the fractional order. In this paper, we choose the fractional Jacobi functions due to their flexibility, ease of calculation, and adaptability. Mathematically, for predefined values of $\alpha$, $\beta$, and $\gamma$, a basis function can be formulated as:
\begin{equation*}
    \phi_{p,q}^{(\gamma)}(\boldsymbol\zeta_p) = \mathcal{J}_q^{(\alpha, \beta)}\left(\varphi_{(0,1; \gamma)}(\sigma(\boldsymbol\zeta_p))\right).
\end{equation*}
In order to allow this mathematical function to be tuned for given data during the optimization process of the networks, we can define these predefined parameters to be learned during training. However, this requires some modifications. For example, to ensure that the Jacobi polynomial converges, its parameters ($\alpha, \beta$) should be greater than $-1$, while a network weight can take any value. To address this, we use the well-known ELU activation function \cite{clevert2015fast} with the $\text{ELU}: \mathbb{R} \to (-\kappa, \infty)$ property:
\begin{equation*}
    \text{ELU}(\zeta; \kappa) = \begin{cases} \zeta & \text{if } \zeta > 0, \\ \kappa \times (e^\zeta - 1) & \text{if } \zeta \leq 0, \end{cases}
\end{equation*}
where $\kappa$ is a parameter that controls the lower bound of the range of the ELU function. As a result of this definition, for parameters $\alpha$ and $\beta$, one can easily set $\kappa$ to $1$ to obtain meaningful Jacobi functions.

Furthermore, the fractional power $\gamma$ needs to be positive to ensure the well-definedness of the Jacobi polynomials. For this purpose, one might consider using functions like $\text{ReLU}(\cdot)$ or $\text{ELU}(\cdot, 0)$. However, these functions allow the parameter to approach infinity, potentially causing instability issues such as Runge's phenomenon \cite{shen2011spectral}. Instead, we propose using the Sigmoid function again to constrain the values of $\gamma$ between zero and one.  

This approach, combined with polynomial degrees, enables the network to find an accurate solution for the given data. By fixing the basis function degree to $n$ (denoted as $\mathcal{J}^{(\alpha,\beta)}_n(\cdot)$), the network can utilize a fractional degree between zero and $n$, exhibiting a "liquid" behavior that permits the network to explore the fractional space for approximating the desired function. Putting all of these together, the new basis function for fKAN can be expressed as:
\begin{equation}
    \begin{aligned}    	
     \phi_{p,q}^{(\gamma)}(\boldsymbol\zeta_p) = \mathcal{J}_q^{(\text{ELU}(\alpha, 1), \text{ELU}(\beta, 1))}\left(\varphi_{(0,1; \sigma(\gamma))}(\sigma(\boldsymbol\zeta_p))\right),
\end{aligned}
\label{eq:JNB}
\end{equation}
with trainable parameters $\alpha$, $\beta$, and $\gamma$. We call this adaptive activation the fractional Jacobi neural block, abbreviated as fJNB. A visualization of the fJNB can be seen in Figure \ref{fig:fJNB}.

The next step in defining fKAN is to establish the summation bounds and the function $\Phi_q(\cdot)$. In the general definition of KANs, the basis functions $\phi_{p,q}(\cdot)$ are local functions (e.g., obtained through B-splines or compact support RBFs), meaning they are non-zero only on a subset of the real line. To approximate a desired function, these local functions must be combined with other local functions in that domain. However, in the case of polynomials, especially orthogonal Jacobi polynomials, the function is non-zero except at some root points (see Theorem \ref{thm:roots}). This allows us to omit the outer summation and focus solely on the inner one. For the function $\Phi_q(\cdot)$, now denoted as $\Phi(\cdot)$, we can use a simple linear function with trainable weights $\theta$ that are adjusted during the network's backward phase. This approach is visualized in Figure \ref{fig:single-fJNB}.

By closely examining the proposed fKAN, one can see that this approach focuses on a specific basis function (e.g., $q=2$), which may not be suitable for a wide variety of applications. For example, in numerical analysis, the Taylor series guarantees that a function can be approximated by a linear combination of monomials. As the number of terms in this series tends to infinity, the approximation becomes increasingly accurate. Formally, for some known weights $\theta_n$:
\begin{equation*}
    \chi(\zeta) = \sum_{n=0}^{\infty} \theta_n \zeta^n.
\end{equation*}
To simulate a similar approach in fKANs, we can use the following formulation:
\begin{equation*}
     \hat{\chi}(\boldsymbol\zeta) =\sum_{q=0}^{Q}\Phi^{(\gamma)}_q\left(\sum_{p=1}^d \phi^{(\gamma)}_{q,p}(\boldsymbol{\zeta}_p)\right).
\end{equation*}
In this formulation, $Q$ represents the maximum degree of the fKAN and $\phi^{(\gamma)}_{q,p}(\boldsymbol{\zeta}_p)$ is defined as in equation \eqref{eq:JNB}. Notably, the outer summation is independent of the input data. Additionally, the function $\Phi_q(\cdot)$ can be defined to fuse the outputs of each fJNB using either a concatenation layer or an attention mechanism. This scheme is illustrated in Figure \ref{fig:multiple-fJNB}

\begin{figure}[ht]
    \centering
    \includegraphics[width=0.6\textwidth]{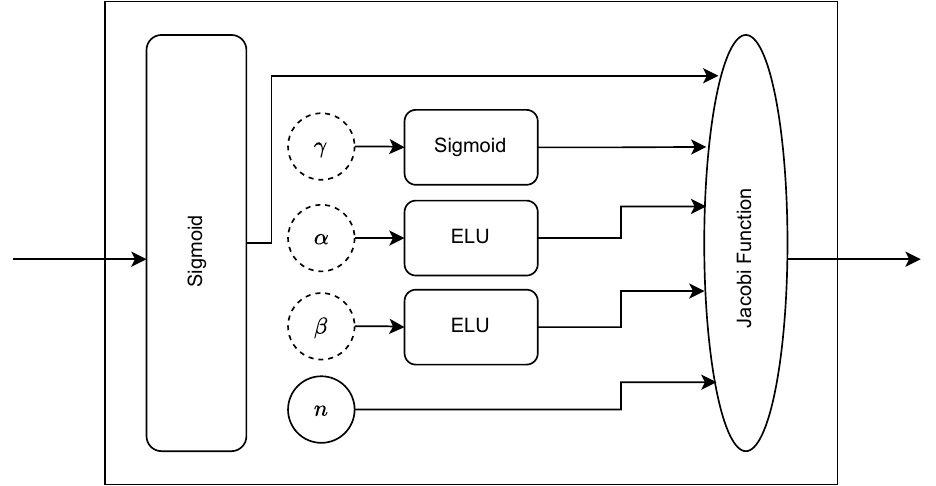}
    \caption{The architecture of the fractional Jacobi neural block proposed in formula \eqref{eq:JNB}, featuring trainable parameters $\alpha$, $\beta$, and $\gamma$. This block serves as an adaptive activation function, enabling the network to learn optimal values for these parameters during training. The fJNB offers flexibility, ease of calculation, and adaptability, making it suitable for various applications in neural network architectures.}
    \label{fig:fJNB}
\end{figure}
\begin{figure}[ht]
    \centering
    \begin{subfigure}[b]{1\textwidth}
    \centering
    \includegraphics[width=1\textwidth]{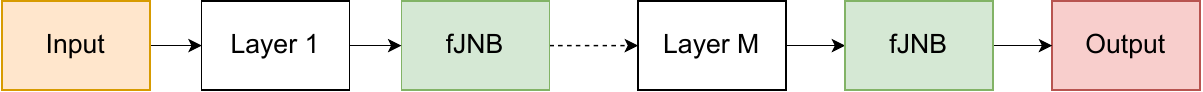}
    \caption{Utilizing fJNB in a sequential architecture, where the value of $q$ may vary for each block. This flexibility allows for the incorporation of different fractional Jacobi basis functions with varying degrees of complexity across successive blocks in the network. By adapting the value of $q$, the network can tailor its representation capabilities to better capture the intricacies of the underlying data.}
    \label{fig:single-fJNB}
  \end{subfigure}\\
  \begin{subfigure}[b]{1\textwidth}
    \centering
    \includegraphics[width=1\textwidth]{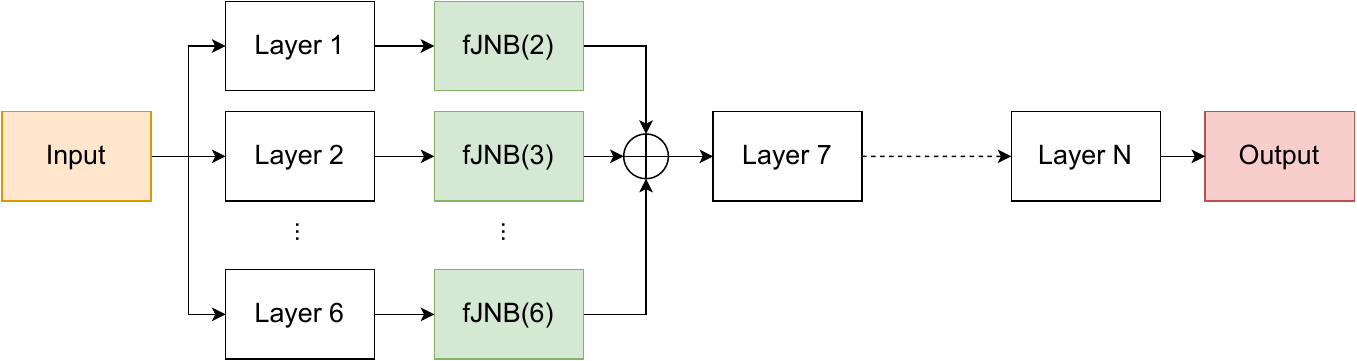}
    \caption{Employing fJNB with varying values of $q$ concurrently, this deep network comprises six subnetworks, each equipped with distinct fJNBs. By incorporating fJNBs with different $q$ values, the network can capture a diverse range of features simultaneously. The outputs of these fJNBs are then fused to produce the final prediction, leveraging the collective information gleaned from multiple representations.}
    \label{fig:multiple-fJNB}
  \end{subfigure}
   
\end{figure}

\section{Experiments}
To evaluate the effectiveness of the proposed fKAN, a series of experiments were designed. In the following sections, we provide detailed descriptions of each experiment, covering a range of tasks including deep learning and physics-informed deep learning. The implementation of fKAN and the subsequent experiments are publicly available in the GitHub repository\footnote{\url{https://github.com/alirezaafzalaghaei/fKAN}}. Additionally, we have created a Python package named \texttt{fkan} to simplify using our proposed method.

All experiments were implemented in Python using the latest version of Keras with TensorFlow as the backend for deep learning tasks and PyTorch for physics-informed tasks. The experiments were conducted on a personal computer equipped with an Intel Core-i3 CPU, 16GB of RAM, and an Nvidia 1650 GPU.

\subsection{Deep learning tasks}
In this section, we will conduct several common deep learning tasks to validate the proposed neural block's effectiveness in handling real-world applications. Specifically, we will assess the accuracy of the proposed fKAN in synthetic regression tasks. Next, we will examine a classification problem in CNNs using the proposed activation function, followed by utilizing a similar CNN network for image denoising tasks. Finally, a one-dimensional CNN (1D-CNN) will be employed for a sentiment analysis task.

\subsubsection{Regression task}
For the first experiment, we consider the following function as the ground truth model:
\begin{equation*}
    \chi(\zeta) := \sin(\pi \zeta) + 10\exp\left(\frac{\zeta}{5}\right) + \frac{\epsilon}{100},
\end{equation*}
where $\epsilon$ is standard normal white noise. We sampled $50$ equidistant points in the domain $[-2,1]$ and split them into train and test sets in a hold-out manner with a test size of $33\%$. To simulate this problem, we designed a single-layer neural network with varying numbers of neurons. The loss function was set to mean squared error, the optimizer to Adam with a learning rate of $0.01$, the maximum iterations to 500, and early stopping with a patience of 200. Figure \ref{fig:compare-activations} shows a comparison between the accuracy of the predictions of this network with different activation functions. As can be seen, the proposed fKAN is more accurate than the well-known activation functions as well as a KAN.

For the next validation task, we compare the accuracy of the network for a multi-layer neural network with a fixed number of five neurons in each layer. The hyperparameters of the network are set the same as in the first experiment. The results are shown in Figure \ref{fig:compare-multilayer} and Table \ref{tbl:multilayer-comparison}. Again, it can be seen that the proposed network performs better than the well-known activation functions and alternatives, although there is a higher variance compared to other activation functions. As in the previous example, the KAN has the lowest variance, which can be attributed to its use of a quasi-Newton optimizer in each trial. A notable issue with the KAN is its increased time complexity as the number of layers increases.

To examine the evolution of the Jacobi polynomial parameters ($\alpha$, $\beta$, and $\gamma$), we selected a single-layer network with a fixed 5 neurons in the hidden layer, increased the number of samples to 250, and removed the early stopping procedure. Figure \ref{fig:Jacobi-parameters} depicts the values of these parameters during the training phase. It can be seen that these trainable parameters converge to optimal values in about 250 iterations.

For the last evaluation, we compare the CPU and GPU time needed to compute each activation function. To do this, we generated a random matrix of shape $(1000, 1000)$ and evaluated different activation functions. The results are shown in Figure \ref{fig:compare-time}. It can be seen that the proposed activation function in CPU implementation is roughly two times slower than the Sigmoid function and three times slower than the fastest activation function, ReLU.

\begin{figure}[!ht]
  \centering
  \begin{subfigure}[b]{0.48\textwidth}
    \centering
    \includegraphics[width=\textwidth, height=5cm]{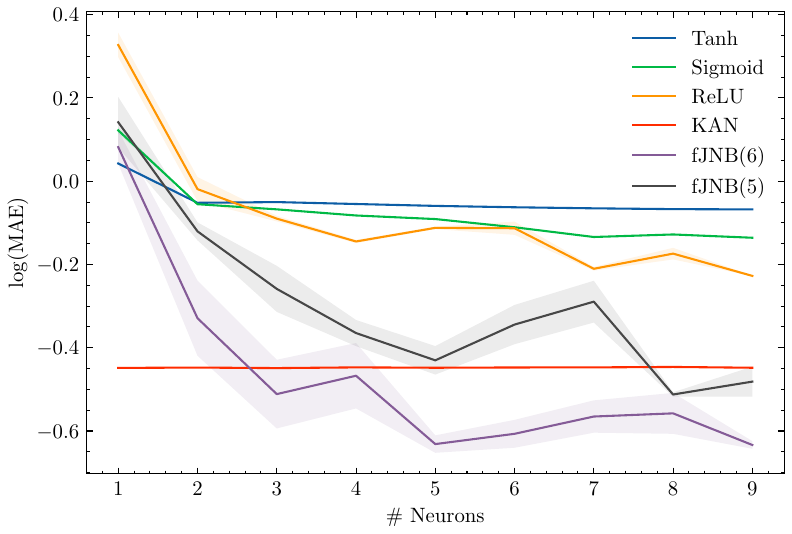}
    \caption{Comparison of prediction accuracy for a single-layer neural network with different activation functions. The proposed activation function demonstrates higher accuracy than well-known functions and KAN.}
    \label{fig:compare-activations}
  \end{subfigure}
  \begin{subfigure}[b]{0.48\textwidth}
    \centering
    \includegraphics[width=\textwidth, height=5cm]{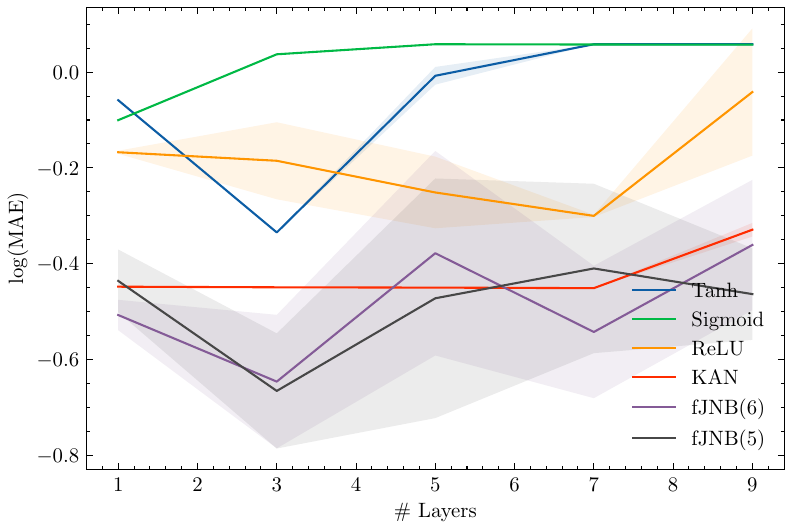}
    \caption{Accuracy comparison for an MLP with a fixed number of neurons per layer. The proposed network performs better than traditional activation functions and alternatives, though with higher variance.}
    \label{fig:compare-multilayer}
  \end{subfigure}%
  \\
  \begin{subfigure}[b]{0.48\textwidth}
    \centering
    \includegraphics[width=\textwidth, height=5cm]{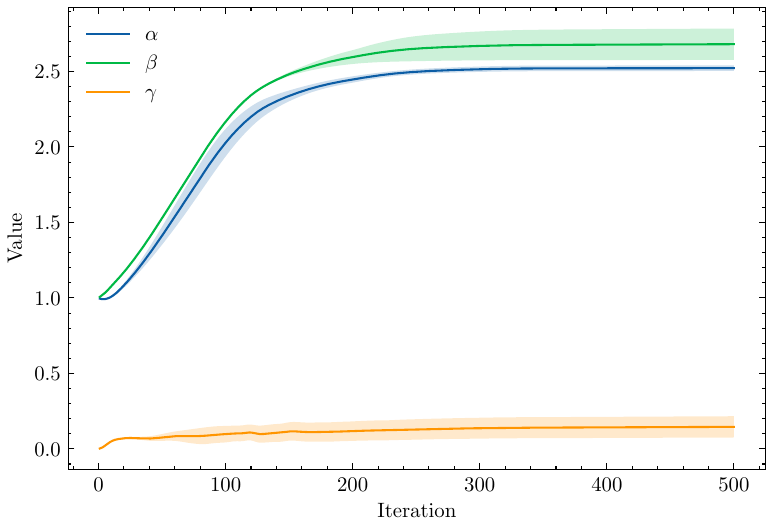}
    \caption{Evolution of Jacobi polynomial parameters ($\alpha$, $\beta$, and $\gamma$) during training. The parameters converge to optimal values after approximately 250 iterations.}
    \label{fig:Jacobi-parameters}
  \end{subfigure}
  \begin{subfigure}[b]{0.48\textwidth}
    \centering
    \includegraphics[width=\textwidth, height=5cm]{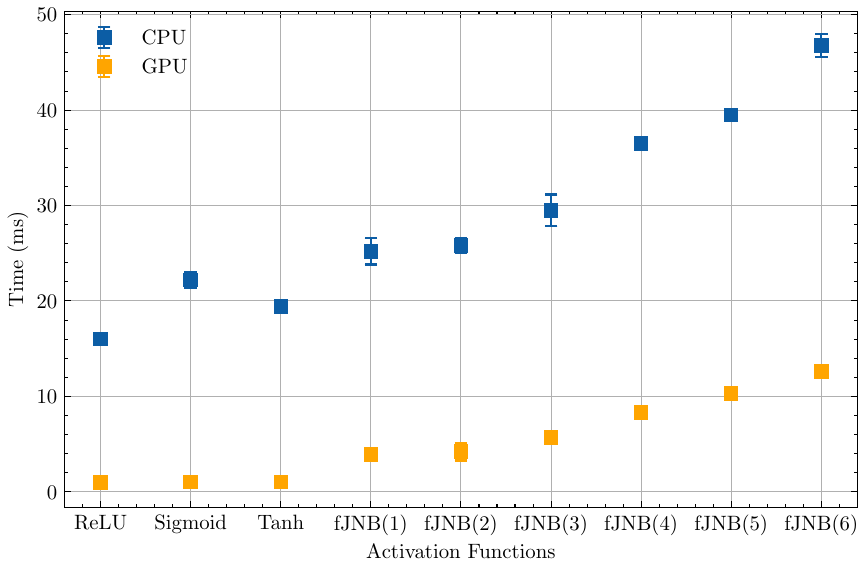}
    \caption{Comparison of CPU and GPU time required to compute various activation functions. The times are the mean of 100 different runs.}
    \label{fig:compare-time}
  \end{subfigure}
  \caption{The results of simulation of a one-dimensional function regression task for different activation functions. The proposed activation function demonstrates higher accuracy than well-known functions as well as KAN.}
  \label{fig:regression-comparison}
\end{figure}

\begin{table}[ht]
\centering
\begin{tabularx}{0.9\textwidth}{@{}*{6}{X}@{}}
\toprule
Act. Func. & \multicolumn{5}{c}{No. of hidden layers}           \\ \midrule
           & 1      & 3      & 5               & 7     & 9      \\ \cmidrule(l){2-6}
SOFTPLUS   & 0.4390 & 0.4630 & 0.4850          & 0.468 & 1.7600 \\
SILU       & 0.6740 & 0.5190 & 0.5100          & 0.337 & 1.0400 \\
SELU       & 0.7410 & 0.9240 & 0.2290          & \textbf{0.171} & 0.1270 \\
Leaky ReLU & 0.3490 & 0.2860 & 0.7060          & 0.152 & 0.3540 \\
GELU       & 0.7130 & 0.6250 & 0.7190          & 0.283 & 0.1240 \\
ELU        & 0.6630 & 0.7520 & 0.4330          & 0.786 & 0.6100 \\
Tanh       & 0.7940 & 0.2640 & 1.2600          & 1.260 & 1.2600 \\
Sigmoid    & 0.7930 & 1.2400 & 1.2600          & 1.260 & 1.2400 \\
ReLU       & 0.4290 & 0.7300 & 0.3430          & 0.353 & 1.7200 \\
JNB(6)     & \textbf{0.0861} & \textbf{0.0186} & 0.0459          & 0.177 & 0.2950 \\
JNB(5)     & 0.6140 & 0.6990 & \textbf{0.0168} & 0.203 & \textbf{0.0218}`' \\ \bottomrule
\end{tabularx}%
\caption{Comparison of mean absolute error for multi-layer neural networks with different activation functions.}
\label{tbl:multilayer-comparison}
\end{table}

\subsubsection{Image classification}

To assess the effectiveness of using generalized Jacobi polynomials as the activation function in fKAN, we utilized the MNIST dataset \cite{lecun1998mnist}, a well-known benchmark for handwritten digit classification in computer vision and deep learning. The MNIST dataset includes a training set of 60,000 examples and a test set of 10,000 examples, consisting of $28\times28$ grayscale images of digits ranging from zero to nine. For training, we used 54,000 samples and reserved 6,000 samples for validation. The data were normalized to a range between zero and one, then fed into a CNN architecture, illustrated in Figure \ref{fig:mnist}, with a batch size of 512. In this architecture, we varied the activation function to be one of Sigmoid, hyperbolic tangent, ReLU, or fJNB. For each activation function, the network was trained using the Adam optimizer with Keras's default learning rate over 30 epochs. The accuracy comparisons of these methods are presented in Table \ref{tbl:mnist} and Figure \ref{fig:mnist-plots}. The results indicate that the fractional Jacobi neural block outperforms the alternatives, even within a convolutional deep learning framework.

\begin{figure}[ht]
        \centering
        \includegraphics[width=1\textwidth]{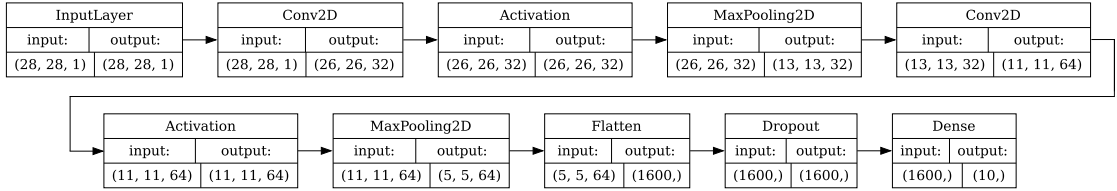}
        \caption{The architecture of proposed method for MNIST classification data.}
    \label{fig:mnist}
\end{figure}

\begin{table}[ht]
\centering

\renewcommand{\tabularxcolumn}[1]{m{#1}}
\begin{tabularx}{0.9\textwidth}{@{}*{5}{X}@{}}

\toprule
Act. Func. & \multicolumn{2}{c}{Loss} & \multicolumn{2}{c}{Accuracy} \\ \midrule
           & Mean        & Std.       & Mean          & Std.         \\ \cmidrule(lr){2-3} \cmidrule(lr){4-5}
Sigmoid    & $0.0611$    & $0.0028$   & $98.092$      & $0.0937$     \\
Tanh       & $0.0322$    & $0.0015$   & $98.904$      & $0.0695$     \\
ReLU       & $0.0256$    & $0.0010$   & $99.140$      & $0.0434$     \\
fJNB(2)     & $0.0252$    & $0.0017$   & $99.134$      & $0.0484$     \\
fJNB(3)     & $0.0224$    & $0.0019$   & $99.200$      & $0.0787$     \\
fJNB(4)     & $\mathbf{0.0217}$    & $0.0008$   & $\mathbf{99.228}$      & $0.0515$     \\
fJNB(5)     & $0.0249$    & $0.0009$   & $99.204$      & $0.0467$     \\
fJNB(6)     & $0.0290$    & $0.0028$   & $99.024$      & $0.1198$     \\ \bottomrule
\end{tabularx}
\caption{Performance of different activation functions in a CNN for classifying MNIST dataset.}
\label{tbl:mnist}
\end{table}

\begin{figure}[ht]
    \begin{subfigure}{0.49\textwidth}
        \includegraphics[width=\linewidth]{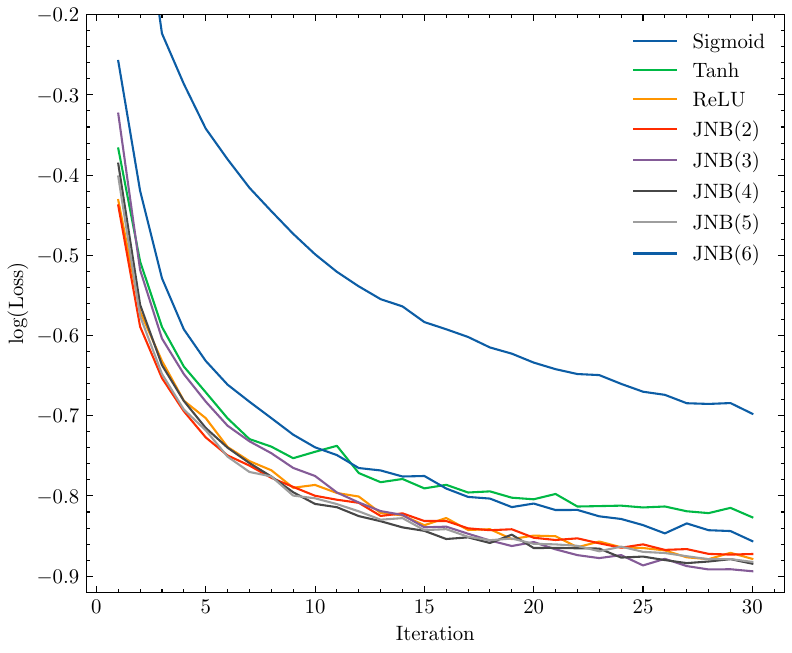}
        \caption{Loss function}
    \end{subfigure}\hfill
    \begin{subfigure}{0.49\textwidth}
        \includegraphics[width=\linewidth]{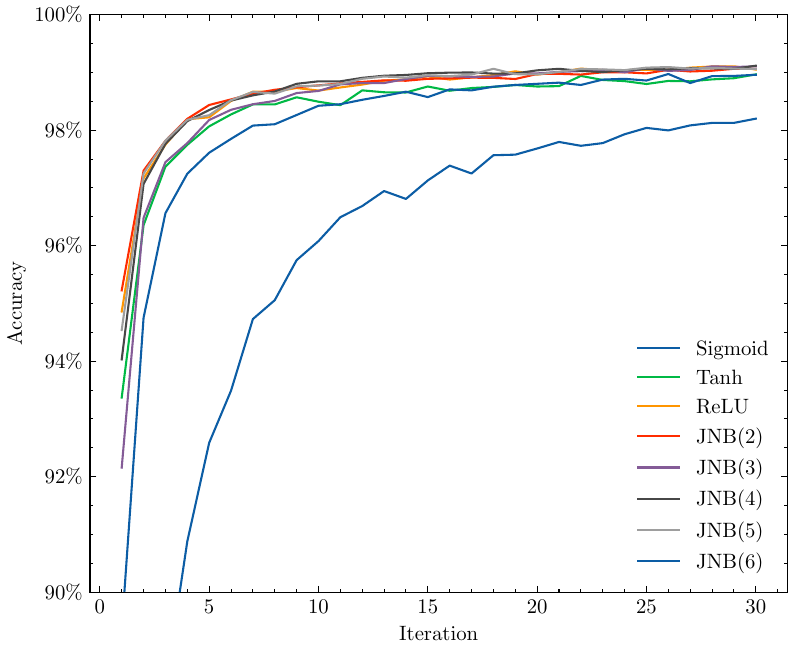}
        \caption{Accuracy}
    \end{subfigure}
    
    \caption{The loss function and accuracy of the proposed method in comparison to well-known activation functions for classifying the MNIST dataset.}
    \label{fig:mnist-plots}
\end{figure}

\subsubsection{Image denoising}

For the next experiment, we develop a CNN tailored for image denoising tasks using the Fashion MNIST dataset as our benchmark. This dataset serves as a more complex replacement for the original MNIST data and includes 60,000 samples of $28\times28$ grayscale images categorized into 10 different fashion classes. As with the MNIST data, before feeding the data into the deep model, we scale the values between zero and one.

The architecture of our CNN for this task is illustrated in Figure \ref{fig:autoencoder}. We begin by pre-training this CNN with the Fashion MNIST training data. This pre-training step is crucial for enabling the network to learn the underlying structure and distribution of the clean data, thus enhancing its understanding of image features, which is beneficial for the subsequent denoising task. After pre-training, we train the network using pairs of noisy and clean images. Noisy images are generated by adding Gaussian noise with a noise factor of 0.3, serving as input, while the corresponding clean images are used as output. This process teaches the network to effectively map noisy images to their denoised versions. In both the pre-training and denoising training phases, we use a batch size of 512 and train the CNN over 20 epochs using the Adam optimizer with Keras's default learning rate.

After the training phase, we evaluate the performance of the trained network by comparing the denoising results of the CNN with the ground truth clean images. The results are quantitatively presented in Table \ref{tbl:autoencoder}, which includes PSNR and SSIM metrics to assess the quality of the denoised images. These metrics are defined as:
\begin{equation*}
\text{PSNR}(\boldsymbol{\zeta}, \hat{\boldsymbol{\zeta}}) = 10 \log_{10}\left(\frac{\text{MAX}^2\{{\boldsymbol\zeta}\} \cdot MN}{\|{\boldsymbol\zeta} - \hat{\boldsymbol{\zeta}}\|_F^2}\right),
\end{equation*}
where $\|\cdot\|_F$ is the Frobenius norm, ${\boldsymbol{\zeta}}, \hat{\boldsymbol{\zeta}} \in [0,1]^{M\times N}$, and
\begin{equation*}
\text{SSIM}(\boldsymbol{\zeta}, \hat{\boldsymbol{\zeta}}) = \frac{(2\mu_{{\zeta}} \mu_{{\hat{\zeta}}} + C_1)(2\sigma_{{\zeta} {\hat{\zeta}}} + C_2)}{(\mu_{{\zeta}}^2 + \mu_{{\hat{\zeta}}}^2 + C_1)(\sigma_{{\zeta}}^2 + \sigma_{{\hat{\zeta}}}^2 + C_2)},
\end{equation*}
in which \(\mu_{\zeta}\) and \(\mu_{\hat{\zeta}}\) are the means of \(\boldsymbol{\zeta}\) and \(\boldsymbol{\hat{\zeta}}\), \(\sigma_{\zeta \hat{\zeta}}\) is the covariance between \(\boldsymbol{\zeta}\) and \(\boldsymbol{\hat{\zeta}}\), \(\sigma_{\zeta}^2\) and \(\sigma_{\hat{\zeta}}^2\) are the variances of \(\boldsymbol{\zeta}\) and \(\boldsymbol{\hat{\zeta}}\), respectively. The constants 
\(C_1\) and \(C_2\) are used to stabilize the division.

\begin{figure}[ht]
    \centering
    \includegraphics[width=1\textwidth]{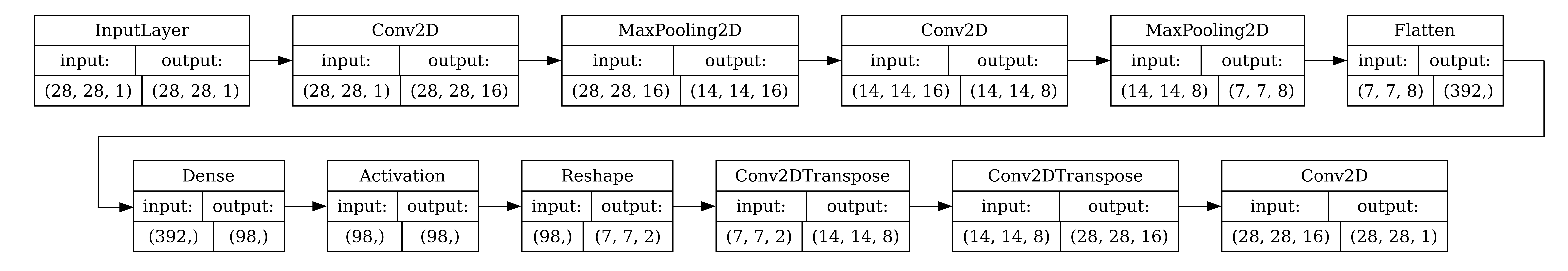}
    \caption{The architecture of the autoencoder used for denoising Fashion MNIST images. One of the activation functions-Sigmoid, tanh, ReLU, or fJNB with \(q\) values ranging from 2 to 6—is selected to allow for flexible adaptation to different non-linearities in the data.}
    \label{fig:autoencoder}
\end{figure}

\begin{table}[ht]
\begin{tabularx}{1\textwidth}{@{} *{7}{>{\centering\arraybackslash}X} @{}}
    \toprule
Act, Func. & \multicolumn{3}{c}{Pre-train}                                                         & \multicolumn{3}{c}{Denoising}                                                \\ \midrule
           & MSE                        & PSNR                        & SSIM                       & MSE                        & PSNR                        & SSIM              \\ \cmidrule(lr){2-4} \cmidrule(lr){5-7}
    Sigmoid    & $0.018 \pm 0.003$          & $18.026 \pm 0.709$          & $0.579 \pm 0.043$          & $0.018 \pm 0.001$          & $17.813 \pm 0.347$          & $0.564 \pm 0.019$ \\
    Tanh       & $0.017 \pm 0.002$          & $18.078 \pm 0.545$          & $0.583 \pm 0.035$          & $0.018 \pm 0.001$          & $17.830 \pm 0.261$          & $0.562 \pm 0.020$ \\
    ReLU       & $0.017 \pm 0.001$          & $18.124 \pm 0.227$          & $0.577 \pm 0.013$          & $0.019 \pm 0.001$          & $17.543 \pm 0.186$          & $0.537 \pm 0.014$ \\
    fJNB(5)     & $0.013 \pm 0.001$          & $\mathbf{19.205} \pm 0.300$          & $0.647 \pm 0.015$          & $0.017 \pm 0.001$          & $18.159 \pm 0.308$          & $0.585 \pm 0.017$ \\
    fJNB(6)     & $\mathbf{0.013} \pm 0.001$          & $19.171 \pm 0.337$          & $\mathbf{0.648} \pm 0.019$          & $\mathbf{0.016} \pm 0.001$          & $\mathbf{18.206} \pm 0.174$          & $\mathbf{0.587} \pm 0.009$ \\ \bottomrule
    \end{tabularx}
    \caption{Comparison of the accuracy of various activation functions for the Fashion MNIST image denoising task.}
    \label{tbl:autoencoder}
\end{table}

\subsubsection{Sentiment Analysis}
For the final deep learning experiment, we evaluate the efficiency of the fJNB in a sentiment analysis task using the well-known IMDB dataset, which includes 20,000 samples for training, 5,000 for validation, and 25,000 for testing. The text data is preprocessed using a custom standardization function to remove HTML tags and punctuation. Subsequently, the text is vectorized with a maximum of 20,000 tokens and a sequence length of 500. This processed data is fed into a deep classifier model, which is a 1D CNN comprising an embedding layer, dropout layers, convolutional layers, global max pooling, a dense layer, and an output layer. The full architecture is depicted in Figure \ref{fig:sentiment}. Similar to the previous examples, we use the Adam optimizer, with a batch size of 512 and 10 epochs for training the network. The performance of the model is evaluated using binary classification metrics, including accuracy, precision, recall, and ROC-AUC, as reported in Table \ref{tbl:sentiment}.

\begin{figure}[t]
    \centering
    \includegraphics[width=1\textwidth]{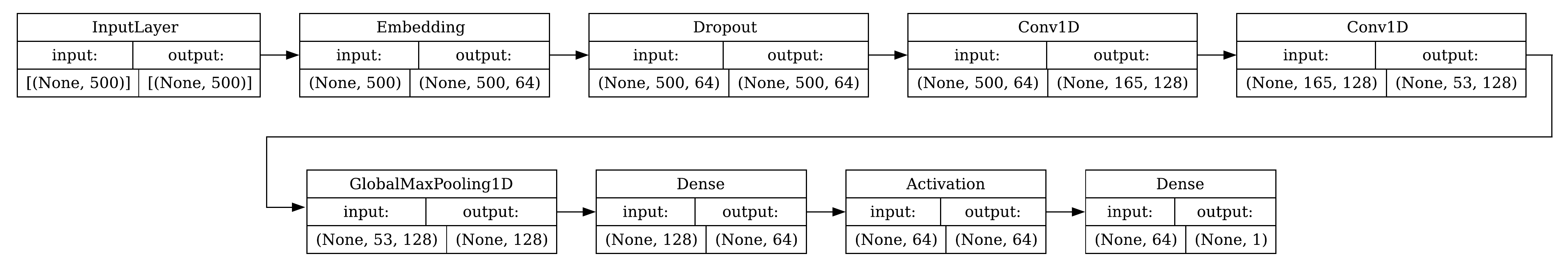}
    \caption{The architecture of the 1D-CNN used for sentiment analysis on the IMDB dataset. One of the activation functions-Sigmoid, tanh, ReLU, or fJNB with \(q\) values ranging from 1 to 6—is selected to allow for flexible adaptation to different non-linearities in the data.}
    \label{fig:sentiment}
\end{figure}

\begin{table}[ht]
\centering
\begin{tabularx}{\textwidth}{@{} *{6}{>{\centering\arraybackslash}X} @{}}
\toprule
Act. Func. & Loss              & Accuracy           & Precision                 & Recall                 & ROC-AUC                \\ \midrule
Sigmoid    & $0.604 \pm 0.039$ & $85.198 \pm 0.468$ & $81.440 \pm 1.511$ & $91.266 \pm 1.532$ & $91.631 \pm 0.422$ \\
Tanh       & $0.610 \pm 0.060$ & $85.604 \pm 0.818$ & $82.364 \pm 2.790$ & $90.912 \pm 2.705$ & $91.894 \pm 0.505$ \\
ReLU       & $0.808 \pm 0.064$ & $83.826 \pm 1.141$ & $78.197 \pm 2.124$ & $\bf{93.995} \pm 1.358$ & $91.080 \pm 0.404$ \\
JNB(1)     & $0.724 \pm 0.074$ & $84.462 \pm 1.156$ & $79.967 \pm 3.160$ & $92.387 \pm 3.090$ & $91.241 \pm 0.501$ \\
JNB(2)     & $0.671 \pm 0.044$ & $85.205 \pm 0.591$ & $80.559 \pm 1.374$ & $92.882 \pm 1.126$ & $91.306 \pm 0.478$ \\
JNB(3)     & $0.646 \pm 0.057$ & $84.828 \pm 0.761$ & $80.100 \pm 1.787$ & $92.808 \pm 1.404$ & $91.576 \pm 0.416$ \\
JNB(4)     & $\bf{0.559} \pm 0.062$ & $\bf{85.645} \pm 0.762$ & $\bf{83.395} \pm 3.549$ & $89.498 \pm 3.916$ & $\bf{92.497} \pm 0.404$ \\
JNB(5)     & $0.691 \pm 0.100$ & $84.247 \pm 1.479$ & $79.283 \pm 3.401$ & $93.190 \pm 2.584$ & $91.454 \pm 0.825$ \\
JNB(6)     & $0.627 \pm 0.107$ & $84.455 \pm 1.658$ & $80.079 \pm 4.017$ & $92.371 \pm 3.215$ & $91.648 \pm 0.985$ \\ \bottomrule
\end{tabularx}
\caption{Binary classification metrics for the IMDB dataset using different activation functions.}
\label{tbl:sentiment}
\end{table}

\subsection{Physics-informed Deep Learning tasks}

To demonstrate the effectiveness of the proposed architecture in solving physical problems, we simulate various types of differential equations within a physics-informed loss function scheme. Mathematically, for a functional operator of the form $\mathcal{N}(\chi) = \mathcal{S}$, we define the residual function as $\mathfrak{R}(\zeta) = \mathcal{N}(\chi)(\zeta) - \mathcal{S}(\zeta)$ and then formulate the loss function of the network as:
\begin{equation*}
    \text{Loss}(\boldsymbol{\zeta}) = \mathfrak{R}(\boldsymbol{\zeta})^\text{T} \mathfrak{R}(\boldsymbol{\zeta}) + \mathfrak{B}^2 + \mathfrak{I}^2,
\end{equation*}
where $\mathfrak{B}$ and $\mathfrak{I}$ represent the network errors for the known boundary and initial conditions, respectively, and $\boldsymbol{\zeta} \in \mathbb{R}^{N \times d}$ is a vector containing $N$ training points in the $d$-dimensional problem domain. To compute the derivatives of the network with respect to the input parameters, we utilize the standard backpropagation algorithm. For fractional-order derivatives, we employ the Caputo operational matrix of differentiation as proposed by Taheri et al. \cite{taheri2024accelerating}. This method approximates the fractional derivative using an L1 finite difference scheme in a matrix format, which is then multiplied by the network output to obtain the fractional derivative of the network.

In the following experiments, we simulate various ordinary, partial, and fractional-order differential equations. Given the complexity of these problems, we utilize the second proposed architecture (Figure \ref{fig:multiple-fJNB}), which incorporates multiple fJNBs simultaneously in a network architecture with a concatenation operator. This approach ensures that different polynomial degrees are included in approximating the function.

\subsubsection{Ordinary differential equations}
For the first example, we consider the well-known Lane-Emden differential equation in its standard form. This benchmark problem is a second-order singular differential equation given by:
\begin{equation*}
    \begin{aligned}
    \frac{\text{d}^2}{\text{d}\zeta^2}\chi(\zeta) &+ \frac2\zeta \frac{\text{d}}{\text{d}\zeta}\chi(\zeta) + \chi^m(\zeta) = 0,\\
    \chi(0)&=1,\quad    \chi'(0)=1,
\end{aligned}
\end{equation*}
where $m$ is an integer parameter, typically ranging from zero to five. For this problem, we use $1500$ equidistant data points as input. Six fJNB blocks are applied, each to a layer containing 10 neurons. The final layer of the network takes the inputs from 60 different activations of the network and outputs a single real value as the approximation for the Lane-Emden equation. Optimization of this network is performed using the limited memory variant of the quasi-Newton Broyden–Fletcher–Goldfarb–Shanno (L-BFGS) algorithm.

The simulation results for this problem, considering $m = 0, 1, \ldots, 5$, are depicted in Figure \ref{fig:lane-emden}. The first root of the predicted function for each value of $m$ holds physical significance and serves as a criterion for accuracy assessment. Comparison with a similar neural network approach \cite{mazraeh2024gepinn} demonstrates significant accuracy improvements (see Table \ref{tbl:lane-emden}).

\begin{table}[ht]
\centering
\begin{tabularx}{\textwidth}{@{}XXXXX@{}}
\toprule
$m$ & {Exact} & {Approximate}  & {Error} & {GEPINN \cite{mazraeh2024gepinn}} \\ \midrule
0   & 2.44948974 & 2.44945454  & $3.52 \times 10^{-5}$          & $\mathbf{1.40 \times 10^{-7}}$\\
1   & 3.14159265 & 3.14160132  & $\mathbf{8.67 \times 10^{-6}}$ & $4.83 \times 10^{-3}$ \\
2   & 4.35287460 & 4.35288394  & $\mathbf{9.34 \times 10^{-6}}$ & $8.93 \times 10^{-3}$ \\
3   & 6.89684860 & 6.89684915  & $\mathbf{5.55 \times 10^{-7}}$ & $1.88 \times 10^{-2}$  \\
4   & 14.9715463 & 14.9712493  & $\mathbf{2.97 \times 10^{-4}}$ & $5.08 \times 10^{-2}$\\ \bottomrule
\end{tabularx}
\caption{Comparison of the first roots of the predicted solution with the exact roots from \cite{horedt2004polytropes} and the approximated results from a similar neural network approach \cite{mazraeh2024gepinn}.
}
\label{tbl:lane-emden}
\end{table}
\begin{figure}[!ht]
    \centering
    \includegraphics[width=0.8\textwidth]{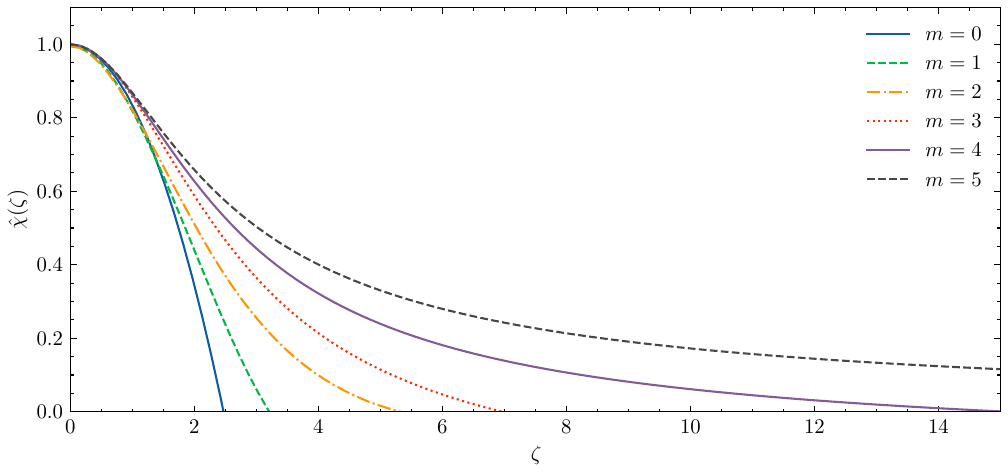}
    \caption{The simulated results of the Lane-Emden differential equation using the fKAN.}
    \label{fig:lane-emden}
\end{figure}

\subsubsection{Partial differential equations}
As a benchmark partial differential equation, we choose the renowned one-dimensional Burgers equation, defined as follows \cite{khater2008chebyshev}:
\begin{equation*}
\begin{aligned}
\frac{\partial}{\partial \tau} \chi(\zeta, \tau) &+ m_0 \chi(\zeta, \tau) \frac{\partial}{\partial \zeta} \chi(\zeta, \tau)  - m_1 \frac{\partial^2}{\partial \tau^2} \chi(\zeta, \tau) = 0, \\
\chi(\zeta,0) &= \frac{m_2}{m_0} + 2\frac{m_1}{m_0} \tanh(x).
\end{aligned}
\end{equation*}
The exact solution to this problem is given by $\chi(\zeta,\tau) = \frac{m_2}{m_0} + 2\frac{m_1}{m_0} \tanh(x - m_2 \tau)$ for predefined parameters $m_0, m_1$, and $m_2$. To address this problem, we utilize a neural network with a similar architecture to the previous example, with the exception of having two hidden neurons per layer before the fJNB and an input dimension of two to accommodate both time and space variables. For these two variables, we fed the network with a Cartesian product of 100 equidistant points in $[0, 1]$. For this example, we compare the predicted solution with the exact solution for two common parameter choices $m_0, m_1$, and $m_2$ as used by Khater et al. \cite{khater2008chebyshev}. The simulation results are depicted in Figure \ref{fig:burgers}, demonstrating accurate predictions.
\begin{figure}[!ht]
  \centering
  \begin{subfigure}{0.46\textwidth}
    \centering
    \includegraphics[width=\textwidth]{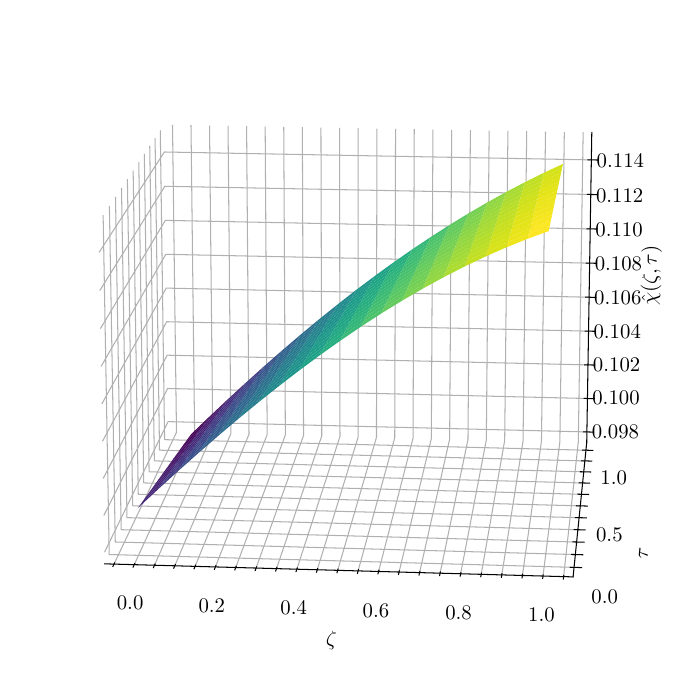}
    \caption{Prediction}
    
  \end{subfigure}%
  \begin{subfigure}{0.46\textwidth}
    \centering
    \includegraphics[width=\textwidth]{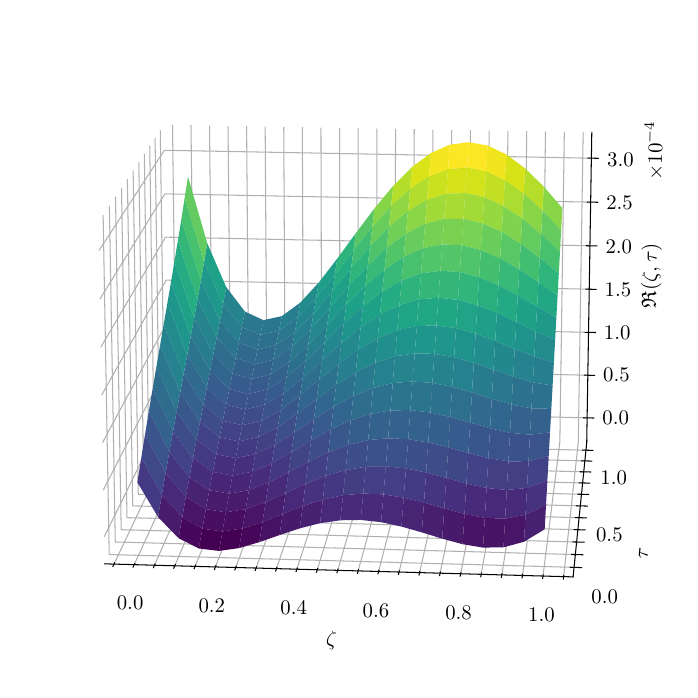}
    \caption{Residual}
  \end{subfigure}
  \\
  \begin{subfigure}{0.46\textwidth}
    \centering
    \includegraphics[width=\textwidth]{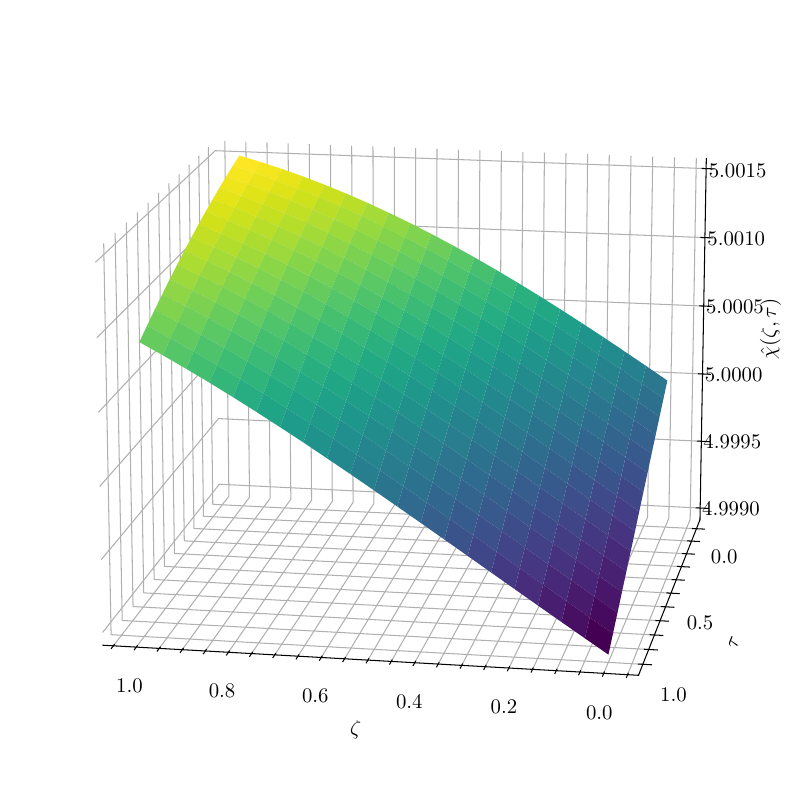}
    \caption{Prediction}
    
  \end{subfigure}%
  \begin{subfigure}{0.46\textwidth}
    \centering
    \includegraphics[width=\textwidth]{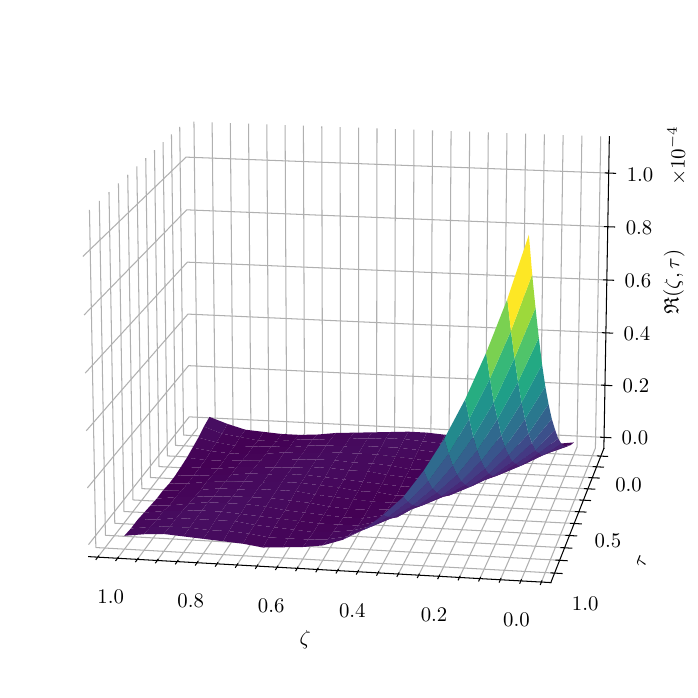}
    \caption{Residual}
    
  \end{subfigure}
  \caption{Simulation results of the Burgers PDE using the fJNB with two sets of parameters: $m_0, m_1, m_2 = 1, 0.01, 0.1$ (top) and $m_0, m_1, m_2 = 0.1, 0.0001, 0.5$ (bottom).}
  \label{fig:burgers}
\end{figure}

\subsubsection{Fractional differential equations}
For the final experiment involving physics-informed neural networks, we investigate a fractional delay differential equation represented as \cite{taheri2023bridging}:
\begin{equation*}
    \begin{aligned}
    \frac{\text{d}^{0.3}}{\text{d}\zeta^{0.3}}\chi(\zeta) &= \chi(\zeta-1) - \chi(\zeta) + 1 - 3 \zeta + 3 \zeta^2+ \frac{2000 \zeta^{2.7}}{1071  \Gamma(0.7)},\\
    \chi(0)&=0,
\end{aligned}
\end{equation*}
where $\Gamma(\cdot)$ denotes the gamma function and the fractional derivative is defined using the Caputo fractional derivative operator \cite{taheri2023bridging, firoozsalari2023deepfdenet}:
\begin{equation*}
    \frac{\text{d}^\alpha}{\text{d}\zeta^\alpha}
\chi(\zeta)=\frac{1}{\Gamma(1-\alpha)}\int_{0}^{\zeta} \frac{\chi'(\tau)}{(\zeta-\tau)^{\alpha}} d\tau, \quad 0<\alpha<1.
\end{equation*}
For simulating this problem, we adopt a similar approach to the previous examples. Specifically, we design the architecture as follows: Each of the first 6 layers connected to fJNBs employs 10 neurons. Subsequently, a concatenation layer is employed. Next, a layer with a weight size of $60 \times 10$ is used to reduce the network features to 10. Then, three fully connected layers are employed, followed by a final regression neuron. The loss function is computed using the residual of this equation:
\begin{equation*}
\mathfrak{R}(\boldsymbol{\zeta}) = \bigg[\mathcal{D}^{0.3}\hat\chi(\boldsymbol\zeta)\bigg] - \bigg[ \hat\chi(\boldsymbol\zeta-1) - \hat\chi(\boldsymbol\zeta) + 1 - 3 \boldsymbol\zeta + 3 \boldsymbol\zeta^2+ \frac{2000}{1071  \Gamma(0.7)} \boldsymbol\zeta^{2.7}\bigg],
\end{equation*}
where $\mathcal{D}^\alpha$ is a lower triangular operational matrix of the derivative for the Caputo derivative defined in Taheri et al. \cite{taheri2024accelerating}. The network weights are then optimized using the L-BFGS algorithm. The predicted solution, the residual with respect to the exact solution, and the network residual loss using this architecture are reported in Figure \ref{fig:delay}. In this example, we observe fluctuations in the residual errors, which, based on experiments, we found are a direct result of the architectures with multiple fJNBs.

\begin{figure}[htbp]
    \centering
    \includegraphics[width=1\textwidth]{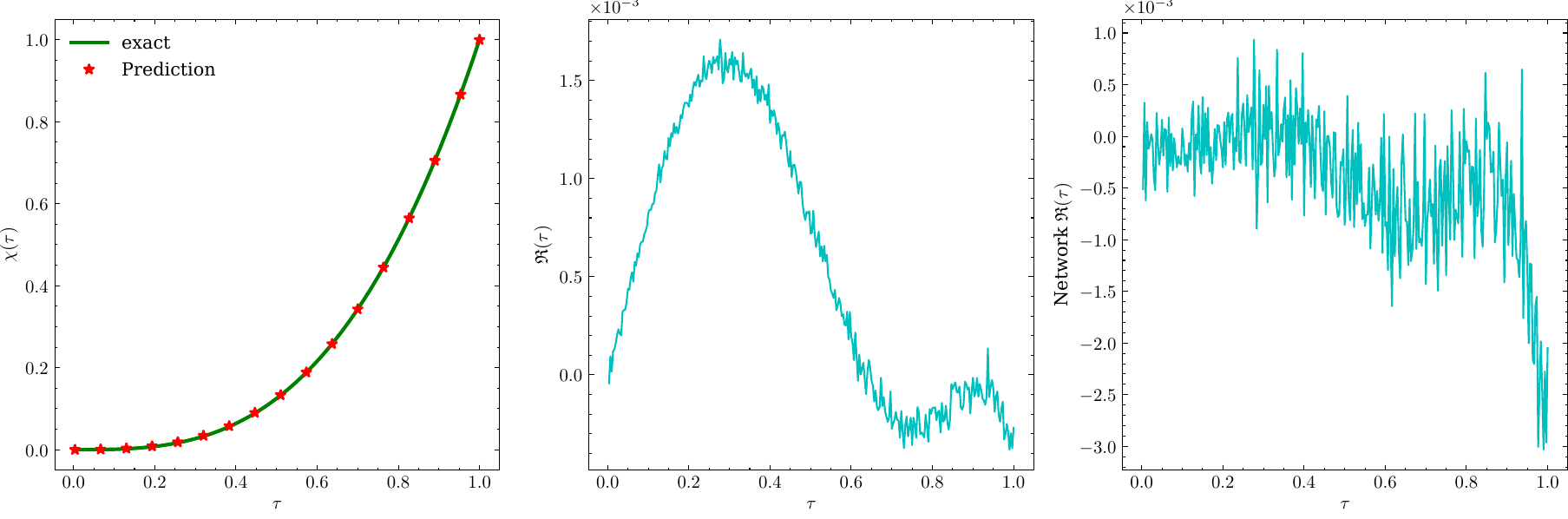}
    \caption{The simulation results for the fractional delay differential equation using the fJNB architecture.}
    \label{fig:delay}
\end{figure}

\section{Conclusion}
This paper introduces a novel extension of the Kolmogorov-Arnold Network framework, incorporating fractional-order orthogonal Jacobi polynomials as basis functions for approximation. Through Equation \eqref{eq:JNB} and Figures \ref{fig:single-fJNB} and \ref{fig:multiple-fJNB}, we demonstrate the integration of this new basis function into the KAN architecture. Comparisons with the Learnable Activation Network \cite{liu2024kan} highlight the similarities and advantages of our approach.

We observe that the proposed basis function exhibits key characteristics of effective activation functions, including non-linearity, simplicity, and straightforward derivatives. Our study further showcases how the fractional order of these polynomials can be leveraged within neural networks, with parameters adaptively tuned during training.

A comprehensive series of experiments across various deep learning tasks, including MLP and 1D/2D-CNN architectures, underscores the superior performance of the fractional KAN over traditional KAN and other activation functions. However, we acknowledge limitations such as increased time complexity compared to simpler activation functions and reduced interpretability relative to KAN due to the global nature of the basis functions. Future work may explore local basis function variants, like fractional B-splines, to address these limitations.

\printbibliography

\end{document}